%% file: main.tex
\documentclass[a4paper,fleqn]{cas-dc}

\usepackage[numbers,sort&compress]{natbib}
\usepackage{color,array,amsthm}
\usepackage{amsmath}
\usepackage{amssymb}
\usepackage{amsfonts}
\usepackage{bm}
\usepackage{mathtools}
\usepackage{booktabs}
\usepackage{siunitx}
\usepackage{algorithm}
\usepackage{algpseudocode}
\usepackage{graphicx}
\usepackage{multirow}
\usepackage{makecell}
\usepackage{url}

\hypersetup{hidelinks,hypertexnames=false}

\input{glyphtounicode}
\newtheorem{theorem}{Theorem}
\newtheorem{assumption}{Assumption}
\newtheorem{lemma}{Lemma}

\newtheorem{remark}{Remark}
\newtheorem{proposition}{Proposition}

\begin{document}

\shorttitle{Disturbance separation and MPC for airborne gimbals}
\shortauthors{Y. Lu et al.}

\title [mode = title]{Physics-Informed and Data-Driven Composite Disturbance Separation and Predictive Control for Airborne Electro-Optical Gimbals with Multiple Constraints}
\tnotemark[1]
\tnotetext[1]{This work was supported by the National Natural Science Foundation of China under Grants U24B20158, 62388101, 62227813, 62595803, T2525027, 62573034, and 62503029; the China Postdoctoral Science Foundation under Grant 2024M764090; the Fundamental Research Funds for the Central Universities under Grant 2242025F20002; and the Fundamental and Interdisciplinary Disciplines Breakthrough Plan of the Ministry of Education of China under Grant JYB2025XDXM206.}

\author[a]{Yaokun Lu}
\ead{yaokunlu@buaa.edu.cn}
\author[a]{Yifan Li}
\ead{yifanLee@buaa.edu.cn}
\author[a]{Zhixuan Zhang}
\ead{by2403080@buaa.edu.cn}
\author[b]{Enmei Wang}
\cormark[1]
\ead{enmeiwang@126.com}
\author[c]{Hao Teng}
\cormark[1]
\ead{hteng@buaa.edu.cn}
\author[a]{Jianzhong Qiao}
\ead{jzqiaobuaa@126.com}

\cortext[1]{Corresponding authors.}

\affiliation[a]{organization={School of Automation Science and Electrical Engineering, Beihang University},
            city={Beijing},
            postcode={100191},
            country={China}}
\affiliation[b]{organization={Hangzhou Innovation Institute of Beihang University},
            city={Hangzhou},
            postcode={311115},
            country={China}}
\affiliation[c]{organization={School of Instrumentation Science and Optoelectronic Engineering, Beihang University},
            city={Beijing},
            postcode={100191},
            country={China}}

\begin{abstract}
This paper studies the constrained pointing control problem of airborne electro-optical gimbals subjected to composite periodic and aperiodic disturbances. Under carrier vibration, input delay, and actuator constraints, achieving line-of-sight stabilization and tracking is challenging. A disturbance-separated model predictive control (MPC) scheme is proposed to compensate these heterogeneous disturbances according to their physical characteristics. An adaptive periodic observer predicts multi-harmonic vibrations over the MPC horizon using linear-regression frequency estimation. A linear observer estimates the aperiodic disturbance, and a bounded long short-term memory (LSTM) network predicts its change over the input delay. Embedding periodic predictions into the MPC model while applying aperiodic compensation outside the optimizer limits amplification of aperiodic estimation errors and improves pointing accuracy under frequency changes and parameter mismatch. On the stated headroom envelope, input-set tightening provides feasible commands within the prescribed limits, and a regional practical tracking bound is established under explicit input-sequence and predicted-state margin conditions. A UAV-mounted two-axis gimbal platform is assembled for comparative indoor hovering experiments, complemented by simulations under frequency changes and parameter mismatch. In the reported trials, the proposed method achieves yaw and pitch relative-angle root-mean-square errors (RMSEs) of $0.770^\circ$ and $0.350^\circ$, reducing tracking errors by 25.1\% and 45.9\% compared with linear active disturbance rejection control.
\end{abstract}

\begin{keywords}
Disturbance separation \sep airborne electro-optical gimbals \sep long short-term memory \sep model predictive control \sep periodic disturbance prediction \sep pointing control
\end{keywords}

\maketitle

\section{Introduction}\label{sec:introduction}

Airborne electro-optical gimbals support emergency response, maritime search and rescue, and bridge crack inspection by directing imaging sensors toward targets of interest. These missions require accurate two-axis line-of-sight (LOS) control, where the gimbal operates as an inertially stabilized platform (ISP) with azimuth and elevation degrees of freedom~\cite{rueStabilizationPrecisionElectrooptical1969, hilkertInertiallyStabilizedPlatform2008, Masten2008, ekstrandEquationsMotionTwoaxes2001}.

Achieving high-accuracy pointing control for such systems is challenged by composite disturbances with distinct physical mechanisms and by multiple constraints. In typical low-to-medium altitude tracking scenarios, atmospheric turbulence and propeller vibration produce high-frequency periodic jitter~\cite{renAngularVibrationSurvey2003, maoModelingRobustContinuous2020, yunInvestigationTwoStageVibration2020}, while rapid slewing excites state-dependent aperiodic torques including gravity, gyroscopic effects~\cite{hilkertInertiallyStabilizedPlatform2008, dongDynamicModelControl2017}, and nonlinear friction~\cite{congIsolationControlInertially2016}. Since these disturbances span distinct spectral bands, a separation strategy is physically motivated. Furthermore, the ISP operates within actuator saturation and field-of-view (FOV) constraints~\cite{dongDynamicModelControl2017}. For aerial vehicles, MPC combined with disturbance observers has been applied to trajectory tracking~\cite{tianAdaptiveModelPredictive2026, nigusieExtendedStateObserver2026}.

Robust and disturbance-observer-based control (DOBC) schemes, including Extended State Observer (ESO) and Active Disturbance Rejection Control, have been extensively studied for ISP pointing~\cite{guoDisturbanceAttenuationRejection2005, chenDisturbanceobserverbasedControlRelated2015, maoAdaptiveNonsingularFast2022, safaRobustOutputFeedback2019, golestaniDisturbanceObserverConstrained2023}. For ISPs under composite disturbances, Lei \textit{et al.}~\cite{leiCompositeControlMethod2015a} combined an adaptive radial basis function neural network with an ESO, and Chang \textit{et al.}~\cite{changCompensationControlStrategy2024} compared adaptive and sliding mode observers for friction and moving-base compensation. Recent work~\cite{zhuangMechatronicDisturbanceSuppression2026, miaoRobustObserverBasedVisual2027} applied ESO-based feedforward to electro-optical platforms and UAV-mounted cameras. However, most DOBC schemes adopt a single total-disturbance estimate, limiting accuracy when periodic and aperiodic components coexist~\cite{guoCompositeDisturbanceFiltering2023}, and require precise disturbance models.

Learning-based techniques augment estimation without explicit models. Sun \textit{et al.}~\cite{sunNeuralNetworkBasedDOBC2017} established a neural DOBC framework, and Shen and Xu~\cite{shenAdaptiveNeuralNetworkBased2021} used adaptive radial basis functions for unmanned-helicopter disturbance rejection. For ISPs, physics-informed neural networks~\cite{leblebiciogluPhysicsInformedDisturbanceEstimation2022} and online learning frameworks~\cite{jiaEVOLVEROnlineLearning2024} improved adaptability, while Huang \textit{et al.}~\cite{huangLSTMMPCDeepLearning2023} combined LSTM with MPC. These methods still use aggregate estimates without structural separation and struggle with high-frequency dynamics due to spectral bias~\cite{rahamanSpectralBiasNeural2019}.

Refined disturbance separation has been proposed to decouple composite disturbances by their influence topology~\cite{guoAntidisturbanceControlTheory2014, zhuDisturbanceSeparationBasedEnhanced2024, yangFixedTimeRefinedDisturbance2024, baoDeepCoupledSlidingMode2025, tengRefinedDisturbanceUtilization2025}, with related studies on antidisturbance controllability and fixed-time predictive control~\cite{yuAntidisturbanceControllabilityAnalysis2021, liuFixedTimeDOBPredictive2024}. Nevertheless, aperiodic estimation in these methods still requires precise physical models.

When disturbance estimates are integrated into MPC, they are commonly held constant over the prediction horizon~\cite{Yang2012CCDC,Maeder2009Automatica}. For periodic disturbances, horizon prediction using a nominal or known period~\cite{Li2014JPC,Xie2026Automatica} or learned profiles~\cite{Ahmed2025LCSS,Zhou2026TIE} has been studied. In airborne ISPs, however, vibration frequency varies with flight state, and propagating an aperiodic observer estimate through the prediction model feeds estimation error into the MPC optimization, which under unmodeled lag can drive actuator saturation (Section~\ref{sec:res_main}). Existing gradient-type frequency adaptation laws~\cite{yangAdaptivePeriodicdisturbanceObserver2023} are nonlinear in the unknown frequency and converge slowly after frequency switches.
Motivated by the above analysis, this paper proposes a disturbance-separated MPC (MPC-DS) framework that combines physics-informed periodic prediction with data-driven delay-residual learning for airborne two-axis ISPs. The proposed disturbance separator decouples the total disturbance into periodic and aperiodic components. Periodic vibration is predicted over the MPC horizon by an APDO with a projected linear-regression frequency estimator, while the aperiodic disturbance is estimated by a biomimetic LSTM-based disturbance observer (LSTM-DO). Periodic predictions enter the MPC prediction model, whereas the aperiodic estimate and its lead correction are compensated outside the optimizer. The MPC formulation incorporates tightened input bounds together with mechanical-range and FOV constraints~\cite{dongDynamicModelControl2017, moratoDataScienceModel2024, liSurveyControlLyapunov2023a, fuAntiWindupMethodBased2024, tengAdaptiveFixedTimeDisturbance2023, kimQuadrotorConstraints2024}.
The principal contributions of this work are summarized as follows:
\begin{list}{\arabic{enumi})}{%
	\usecounter{enumi}%
	\setlength{\leftmargin}{1.6em}%
	\setlength{\labelwidth}{1.1em}%
	\setlength{\labelsep}{0.5em}%
	\setlength{\topsep}{0pt}%
	\setlength{\partopsep}{0pt}%
	\setlength{\itemsep}{0pt}%
	\setlength{\parsep}{0pt}%
}
	\item \emph{Coupled gimbal--disturbance modeling:} A coupled Newton--Euler dynamic model characterizes how carrier motion, two-axis kinematics, and multi-source disturbances---gyroscopic, gravitational, cross-axis coupling, and LuGre friction torques---enter the azimuth and elevation channels, motivating decomposition into a retained multi-harmonic component and a complementary residual component for targeted observer design.
	\item \emph{Heterogeneous disturbance separation:} A biomimetic separator integrates an APDO, whose linear-regression frequency law with projection keeps the estimate in the admissible band and follows frequency switches, and a data-driven LSTM-DO for the aperiodic torque, whose projected outputs bound the network effect on the estimation error. The LSTM-DO employs a prediction-correction architecture: the linear observer corrects the current value, and the network predicts the delay-induced change that linear lead correction misses.
	\item \emph{Constraint-aware disturbance-separated MPC:} A composite predictive control architecture extrapolates the periodic estimate over the prediction horizon using the online frequency estimate while compensating for the aperiodic estimate outside the optimizer with delay-lead correction. The input-set construction gives feasible bounded commands on the stated headroom envelope. A regional input-to-state practical tracking bound is established under explicit full-sequence input margins, predicted-state margins and bounded implementation residuals.
\end{list}

The remainder of this paper is structured as follows. Section~\ref{sec:modeling} details the dynamic modeling of the dual-axis ISP. Section~\ref{sec:control_design} elaborates on the proposed composite anti-disturbance control architecture, including stability analysis. Section~\ref{sec:results} presents simulation and experimental results to evaluate the approach, and Section~\ref{sec:conclusion} concludes the paper.

\section{Deep-Coupled Dynamic Modeling}\label{sec:modeling}

This section derives the coupled dynamic model of the two-axis ISP using the Newton--Euler formulation, proceeding from rigid-body kinematics to the state-space equations. It then characterizes the periodic and aperiodic disturbance components---including gravitational, gyroscopic, airframe-vibration, friction, and coupling torques---and formulates the physical and mission constraints.

\subsection{Kinematics and Dynamics Modeling}\label{sec:kinematics}

The ISP is equipped with a gyroscope, motor encoders, and an electro-optical camera. The mechanism consists of an outer azimuth gimbal (body A) rotating about the $z$-axis by $\theta_a$ and an inner elevation gimbal (body E) rotating about the $y$-axis by $\theta_e$, with the camera LOS aligned with the elevation platform. By the Newton--Euler formulation and projection onto the joint axes, the manipulator-form dynamics are (detailed derivation in Supplement~\ref{S-S01}):
\begin{equation}
	\label{eq:manipulator_form}
	\begin{bmatrix}
		M_{11} & M_{12} \\
		M_{21} & M_{22}
	\end{bmatrix}
	\begin{bmatrix}
		\ddot{\theta}_a \\
		\ddot{\theta}_e
	\end{bmatrix}
	+
	\begin{bmatrix}
		C_{11} & C_{12} \\
		C_{21} & C_{22}
	\end{bmatrix}
	\begin{bmatrix}
		\dot{\theta}_a \\
		\dot{\theta}_e
	\end{bmatrix}
	=
	\begin{bmatrix}
		\tau_a \\
		\tau_e
	\end{bmatrix}
	+
	\begin{bmatrix}
		d_a \\
		d_e
	\end{bmatrix},
\end{equation}
where $\ddot{\theta}_a$ and $\ddot{\theta}_e$ denote the joint angular accelerations, and $[\tau_a,\tau_e]^\mathrm{T}$ and $[d_a,d_e]^\mathrm{T}$ denote the control torque vector and disturbance torque vector, respectively. The matrices $\bm M$ and $\bm C$ collect the configuration-dependent inertia and velocity-coupling terms, with $M_{12}=M_{21}=C_{22}=0$ for this model. Their element expressions are given in Supplement~\ref{S-S01}.
Defining the state vector $\bm{x} = [\theta_a, \dot{\theta}_a, \theta_e, \dot{\theta}_e]^\mathrm{T}$, the explicit state-space dynamics are expressed as:
\begin{equation}\label{eq:state_space_physical}
	\left\{\begin{aligned}
		\dot{x}_1&=x_2,\\
		\dot{x}_2&=J_a^{-1}\left(\tau_a +H_a + d_a\right),\\
		\dot{x}_3&=x_4,\\
		\dot{x}_4&=J_e^{-1}\left(\tau_e +H_e + d_e\right),
	\end{aligned}\right.
\end{equation}
where $J_a = M_{11}$, $J_e = M_{22}$, $H_a = -C_{11}x_2 - C_{12}x_4$, and $H_e = -C_{21}x_2$. Here, $x_1 = \theta_a$, $x_2 = \dot{\theta}_a$, $x_3 = \theta_e$, and $x_4 = \dot{\theta}_e$ represent the azimuth angle, azimuth angular velocity, elevation angle, and elevation angular velocity, respectively. For the speed-command implementation, Section~\ref{sec:sim_setup} uses an identified input-equivalent model, with the input and disturbance expressed in command levels.

\subsection{Disturbance Modeling and Analysis}\label{sec:disturbance_modeling}

Based on the kinematic relations and Newton--Euler dynamics, this subsection characterizes the axis disturbances $d_a$ and $d_e$. With the subscript $i\in\{a,e\}$ denoting the azimuth or elevation axis, the total disturbance $d_i$ comprises gravitational, gyroscopic, airframe-vibration, and friction torques, together with unmodeled dynamics:
\begin{equation}\label{eq:total_disturbance_def}
	d_i = -f_{g,i} - f_{gyro,i} - f_{v,i} - f_{d,i} + \delta_i.
\end{equation}
Each component is analyzed below; detailed derivations are in Supplement~\ref{S-S02}.

Gravitational and gyroscopic torques depend on mass imbalance and carrier attitude and angular rates. Because carrier measurements arrive with delay and a lower update rate than the gimbal loop, these contributions are treated as unknown~\cite{MIL-STD-1553B,Masten2008}.

\textit{Vibration Torque $f_{v,i}$}: A multi-harmonic model is adopted to describe the aircraft vibration disturbances~\cite{wangDisturbanceObservationSuppression2024}:
\begin{equation}\label{eq:harmonic_disturbance}
	\left\{
	\begin{aligned}
		f_{v,a} &= \gamma_{a,0} + \sum_{j=1}^{N_a} \gamma_{a,j} \sin(\Omega_{a,j} t + \varphi_{a,j}), \\
		f_{v,e} &= \gamma_{e,0} + \sum_{j=1}^{N_e} \gamma_{e,j} \sin(\Omega_{e,j} t + \varphi_{e,j}),
	\end{aligned}
	\right.
\end{equation}
where $\gamma_{i,0}$ denotes the constant vibration bias, $N_i$ denotes the number of dominant modes, $\gamma_{i,j}$ is the vibration amplitude, $\Omega_{i,j}$ is the harmonic frequency, and $\varphi_{i,j}$ is the phase offset. These modal parameters depend on the aircraft flight state. The APDO below retains the modes described by one common fundamental and its integer harmonics. Constant bias and modes outside this retained family are assigned to the residual component.

The LuGre friction component~\cite{canudasdewitNewModelControl1995a} exhibits velocity-dependent transitions; its uncertain parameters are treated as unknown in the controller. The residual $\delta_i$ includes cable restraint, higher-order coupling and sensor bias.

Based on their physical origins, the disturbance components fall into two categories. The vibration-induced torque $f_{v,i}$, which originates from external airframe excitations (propellers and atmospheric turbulence), is independent of the gimbal states and exhibits a multi-harmonic structure. Conversely, the gravitational torque $f_{g,i}$, gyroscopic torque $f_{gyro,i}$, and friction torque $f_{d,i}$ depend directly on the system and carrier states: $f_{g,i}$ and $f_{gyro,i}$ vary with the aircraft attitude $(\alpha, \beta)$ and angular velocity $\bm{\omega}_B$, whereas $f_{d,i}$ varies nonlinearly with the joint angular velocity $\dot{\theta}_i$. Accordingly, the total disturbance $d_i$ is decomposed into a periodic component $d_{h,i}$ and an aperiodic component $d_{r,i}$:
\begin{equation}\label{eq:dist_group}
	d_{i} = d_{h,i} + d_{r,i}, \quad
	\left\{
	\begin{aligned}
		d_{h,i} &= - f^{\mathrm{harm}}_{v,i}, \\
		d_{r,i} &= -f_{g,i} - f_{gyro,i} - f_{d,i}\\
&\quad- f^{\mathrm{res}}_{v,i}+\delta_i.
	\end{aligned}
	\right.
\end{equation}

Here $f_{v,i}=f^{\mathrm{harm}}_{v,i}+f^{\mathrm{res}}_{v,i}$, with the bias included in $f^{\mathrm{res}}_{v,i}$. The retained component $d_{h,i}$ is represented by a common phase and integer harmonics, making it suitable for frequency estimation and horizon extrapolation. The residual channel $d_{r,i}$ includes the state-dependent gravity, gyroscopic and friction terms together with the unretained vibration remainder. Its dominant components vary over different temporal scales: $f_{g,i}$ and $f_{gyro,i}$ vary relatively slowly with aircraft attitude maneuvers, whereas $f_{d,i}$ changes rapidly with the joint velocity $\dot{\theta}_i$. The analysis uses bounded amplitudes and sampled residual increments on the operating envelope. Smooth frequency-drift estimates apply between frequency switches; a separate amplitude bound covers predictions across a switch.

\subsection{Constraint Analysis}\label{sec:constraints}

The ISP operates under both physical hardware limits and mission tracking constraints. The physical limits comprise actuator torque saturation and joint mechanical range limits, while the mission requirement imposes a camera FOV constraint so that the target remains within the imaging frame.

The actuator saturation constraint can be represented by the bounded control torque of the motors, where $\bm{\tau}_{\mathrm{min}}$ and $\bm{\tau}_{\mathrm{max}}$ denote the minimum and maximum allowable torque limits:
\begin{equation}\label{eq:constraint_input}
	\bm{\tau}_{\mathrm{min}} \le \bm{\tau} \le \bm{\tau}_{\mathrm{max}}.
\end{equation}

The gimbal mechanical structure limits its angular motion, including $\theta_a, \dot{\theta}_a, \theta_e$ and $\dot{\theta}_e$. Then the gimbal mechanical range constraint can be collectively represented as state constraints, bounded by the minimum limit $\bm{x}_{\mathrm{min}}$ and maximum limit $\bm{x}_{\mathrm{max}}$:
\begin{equation}\label{eq:constraint_state_physical}
	\bm{x}_{\mathrm{min}} \le \bm{x} \le \bm{x}_{\mathrm{max}}.
\end{equation}

With the state variables defined in Section~\ref{sec:kinematics} ($x_1=\theta_a, x_3=\theta_e$), the camera's FOV constraint is expressed as:
\begin{equation}\label{eq:constraint_state_fov}
	\begin{aligned}
		|x_1(t) - \theta_{a,\mathrm{ref}}(t)| &\le \theta_{\mathrm{FOV},a} / 2, \\
		|x_3(t) - \theta_{e,\mathrm{ref}}(t)| &\le \theta_{\mathrm{FOV},e} / 2,
	\end{aligned}
\end{equation}
where $\theta_{\mathrm{FOV},i}$ is the angular width of the camera's FOV, and $\theta_{i,\mathrm{ref}}$ is the camera's reference LOS in the corresponding axis.

\section{Composite Anti-Disturbance Control}\label{sec:control_design}

This section develops the disturbance-separated predictive control architecture based on the dynamic model and disturbance decomposition introduced above. To facilitate controller design, the dynamics in \eqref{eq:state_space_physical} are written in the compact control-affine form
\begin{equation}\label{eq:state_space_affine}
	\dot{\bm{x}}(t) = \bm{f}(\bm{x}) + \bm{g}(\bm{x})(\bm{\tau}(t) + \bm{d}(t)),
\end{equation}
where $\bm{\tau}=[\tau_a, \tau_e]^\mathrm{T}$ is the control torque vector and $\bm{d}=[d_a, d_e]^\mathrm{T}$ denotes the total disturbance vector modeled in Section~\ref{sec:disturbance_modeling}. The drift vector $\bm{f}(\bm{x})$ and input matrix $\bm{g}(\bm{x})$ are given by
\begin{equation}\label{eq:affine_matrices}
	\begin{aligned}
	\bm{f}(\bm{x}) &=
	\begin{bmatrix}
		x_2 \\ -J_a^{-1}(C_{11}x_2 + C_{12}x_4) \\ x_4 \\ -J_e^{-1}C_{21}x_2
	\end{bmatrix},\\
	\bm{g}(\bm{x}) &=
	\begin{bmatrix}
		0 & 0 \\ J_a^{-1} & 0 \\ 0 & 0 \\ 0 & J_e^{-1}
	\end{bmatrix}.
	\end{aligned}
\end{equation}

The proposed architecture processes the periodic and aperiodic disturbance components through separate paths, as shown in Fig.~\ref{fig:control_architecture}. The periodic component is \emph{predicted} over the MPC horizon and injected into the prediction model, whereas the aperiodic component is compensated outside the optimizer by certainty equivalence together with the known coupling term $H_i$. For axis $i$ at step $k$, the applied control torque is
\begin{equation}\label{eq:total_control}
	\tau_{i,k} = v_{i,k} - H_{i,k} - \hat{d}^{\mathrm{c}}_{r,i,k},
\end{equation}
where $v_{i,k}$ is the nominal command from the MPC in Section~\ref{sec:mpc}, $H_{i,k}=H_i(\bm{x}_k)$ is evaluated at the measured state, and $\hat{d}^{\mathrm{c}}_{r,i,k}$ is the bounded aperiodic estimate of the LSTM-DO in Section~\ref{sec:lstm_do}. The periodic prediction $\hat{d}_{h,i}(k+j|k)$ of the APDO in Section~\ref{sec:apdo} is used only inside the MPC. With $H_i$ compensated, the two axes are decoupled and are treated separately; in the rest of this section the axis subscript $i$ is dropped where no confusion arises, and $T_s$ denotes the sampling period. The observers and the MPC use a constant nominal inertia $J$; for the azimuth axis, $M_{11}$ depends on $\theta_e$, and the torque $(J-M_{11})\ddot\theta_a$ is absorbed into $d_r$, which is bounded on the operating envelope.

\begin{figure*}[!htbp]
	\centering
	\includegraphics[width=\textwidth]{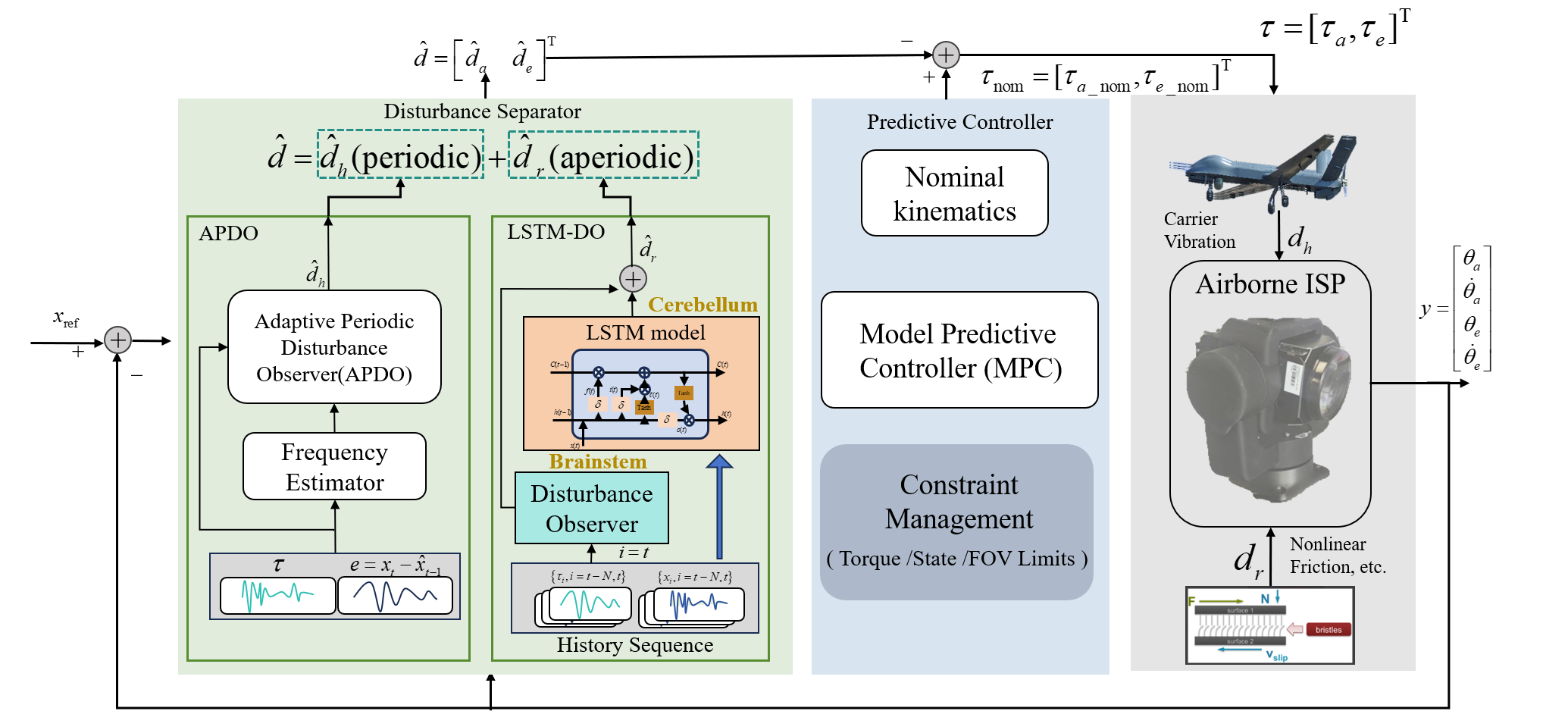}
	\caption{Proposed MPC architecture with physics-informed and data-driven disturbance separation.}
	\label{fig:control_architecture}
\end{figure*}

\subsection{APDO Design}\label{sec:apdo}

As established in Section~\ref{sec:disturbance_modeling}, the periodic component $d_{h,i}$ has a multi-harmonic structure whose fundamental frequency varies with the flight state. Although adaptive neural-network methods can estimate periodic disturbances~\cite{maApproximationbasedNussbaumGain2022, yangDOBasedAdaptiveConsensus2023}, constrained predictive control of an airborne gimbal requires online frequency tracking and multi-step prediction of $d_{h,i}$ over the MPC horizon. The proposed APDO therefore combines a frequency estimator with a harmonic predictor. Unlike the gradient-based frequency law in~\cite{yangAdaptivePeriodicdisturbanceObserver2023}, which is nonlinear in the unknown frequency and analyzed for a single sinusoid, the estimator below is linear in the unknown parameter and projects the frequency estimate onto the admissible band at every step.

\textit{Posterior disturbance.} Both observers are driven by the posterior disturbance obtained from the discretized inverse dynamics,
\begin{equation}\label{eq:posterior}
	d^+_{k} = J\,\frac{\omega_k-\omega_{k-1}}{T_s} - \tau_{k-1-n_d} - H_{k-1},
\end{equation}
where $\omega=\dot\theta$ and $n_d\ge0$ is the input delay in samples ($n_d=0$ for the model \eqref{eq:state_space_physical}).

\textit{Frequency estimation.} Passing $d^+_k$ through a second-order resonator centered at the geometric midpoint of the admissible band $[\Omega_{\min},\Omega_{\max}]$ yields the filtered signal $s_k$. With the integer lag $D=\mathrm{round}\big(\pi/(2\Omega_{\mathrm{nom}}T_s)\big)$ chosen such that $\Omega_{\max}T_sD<\pi$, every sampled sinusoid $s_k=A\cos(\Omega T_sk+\varphi)$ satisfies the exact trigonometric identity
\begin{equation}\label{eq:lin_identity}
	s_k+s_{k-2D}=a\,s_{k-D},\qquad a=2\cos(\Omega T_sD),
\end{equation}
so the unknown enters linearly. Let $m_k=(1-\kappa_m)m_{k-1}+\kappa_m s_{k-D}^2$, $\kappa_m\in(0,1)$, $m_0>0$, and define the normalized pair $y_k=(s_k+s_{k-2D})/\sqrt{m_k}$, $\phi_k=s_{k-D}/\sqrt{m_k}$. Since $m_k\ge\kappa_m s_{k-D}^2$, $\phi_k^2\le\phi_{\max}^2:=1/\kappa_m$ for every input. The estimate is updated by scalar recursive least squares (RLS) with forgetting factor $\lambda\in(0,1)$ followed by projection,
\begin{equation}\label{eq:apdo_freq_estimator}
\begin{aligned}
g_k&=\frac{P_{k-1}\phi_k}{\lambda+P_{k-1}\phi_k^2},\\
\hat a_k&=\Pi_{\mathcal A}\big(\hat a_{k-1}+g_k(y_k-\phi_k\hat a_{k-1})\big),\\
P_k&=\Pi_{[p_{\min},p_{\max}]}\big((P_{k-1}-g_k\phi_kP_{k-1})/\lambda\big),
\end{aligned}
\end{equation}
with $\mathcal A=[a_{\min},a_{\max}]$, where $a_{\min}=2\cos(\Omega_{\max}T_sD)$ and $a_{\max}=2\cos(\Omega_{\min}T_sD)$, $0<p_{\min}<p_{\max}$, and $\hat\Omega_k=\arccos(\hat a_k/2)/(T_sD)$. Because $\Omega_{\max}T_sD<\pi$, the map $\Omega\mapsto a$ is a bijection between $[\Omega_{\min},\Omega_{\max}]$ and $\mathcal A$, so $\hat\Omega_k\in[\Omega_{\min},\Omega_{\max}]$ for every $k$.

\begin{lemma}[Frequency estimation]\label{lem:freq}
Let $a^\star_k=2\cos(\Omega_kT_sD)$ and write $y_k=\phi_ka^\star_k+e_k$, where the equation error $e_k$ collects out-of-band resonator residuals, measurement noise, and filter transients. Suppose $|e_k|\le\bar e$, $|a^\star_{k+1}-a^\star_k|\le\bar\delta_a$, and the normalized regressor is persistently exciting such that $\sum_{l=0}^{N-1}\phi_{k+l}^2\ge\alpha_0>0$ for all $k$. Then, with $\tilde a_k=a^\star_k-\hat a_k$:
(i) $|\tilde a_{k+N}|\le\rho_a|\tilde a_k|+N(\bar g\bar e+\bar\delta_a)$, where
$\rho_a=\exp\!\big(-p_{\min}\alpha_0/(\lambda+p_{\max}\phi_{\max}^2)\big)<1$ and $\bar g=\tfrac12\sqrt{p_{\max}/\lambda}$; hence
$\limsup_{k\to\infty}|\tilde a_k|\le N(\bar g\bar e+\bar\delta_a)/(1-\rho_a)$, and the convergence is exponential when $\bar e=\bar\delta_a=0$;
(ii) $|\tilde\Omega_k|\le L_f|\tilde a_k|$ with $L_f=\big(2T_sD\sqrt{1-a_{\mathrm{pk}}^2/4}\big)^{-1}$, $a_{\mathrm{pk}}=\max(|a_{\min}|,|a_{\max}|)$;

(iii) regardless of excitation, $\hat\Omega_k\in[\Omega_{\min},\Omega_{\max}]$ and $|\tilde\Omega_k|\le\Omega_{\max}-\Omega_{\min}$ for all $k$.
\end{lemma}

\begin{proof}
Substituting $y_k=\phi_ka^\star_k+e_k$ into \eqref{eq:apdo_freq_estimator}, the pre-projection error obeys $\tilde a^-_{k}=(1-g_k\phi_k)\big(\tilde a_{k-1}+a^\star_k-a^\star_{k-1}\big)-g_ke_k$. Because $P_{k-1}\in[p_{\min},p_{\max}]$, $g_k\phi_k\in[0,1)$, $|g_k|\le\bar g=\tfrac12\sqrt{p_{\max}/\lambda}$, and projection is non-expansive, iterating over one window with $\ln(1-x)\le-x$ gives the window inequality; iterating it yields the ultimate bound. Property (ii) follows from the Lipschitz constant of $\arccos(a/2)$ on $\mathcal A$, and (iii) from the projection and bijection. For details see Supplement~\ref{S-S05}.
\end{proof}

For a slowly drifting frequency, $\bar\delta_a\le2T_sD\cdot T_s\delta_\Omega$ with $\delta_\Omega=\sup|\dot\Omega|$. Property (iii) holds even when the periodic disturbance vanishes. In that case, normalization amplifies noise, but the interval projection keeps $\hat\Omega_k$ within the admissible band and the amplitude projection below keeps the harmonic prediction bounded.

\textit{Harmonic prediction.} Let $\hat\psi_k=\hat\psi_{k-1}+\hat\Omega_kT_s$ and $\bm\varphi(\psi)=[\cos\psi,\sin\psi,\ldots,\cos n_h\psi,\sin n_h\psi]^{\mathrm T}$. Write the retained disturbance as $d_h(t)=\bm\varphi(\psi(t))^{\mathrm T}\bm b(t)$, where $\dot\psi=\Omega(t)$ and $\|\bm b(t)\|\le A_{\max}$. Rotating the coefficient vector to the phase reference $\hat\psi_k$ gives
$d_h(t_{k+j})=\bm\varphi(\hat\psi_k+\psi^\star_{k,j})^{\mathrm T}\bm\vartheta^\star_{k,j}$, with $\psi^\star_{k,j}=\int_{t_k}^{t_{k+j}}\Omega(t)\,dt$ and $\|\bm\vartheta^\star_{k,j}\|\le A_{\max}$. On a smooth segment, assume $\|\bm\vartheta^\star_{k,j}-\bm\vartheta^\star_{k,0}\|\le jT_s\delta_b$; constant amplitudes give $\delta_b=0$. The harmonic coefficients are fitted to $d^+_k$ by RLS, with the regressors passed through the same posterior low-pass filter. An additional constant regressor absorbs slow components but is excluded from prediction. The nonconstant coefficients are projected onto $\|\hat{\bm\vartheta}_k\|\le A_{\max}$, and
\begin{equation}\label{eq:apdo_q_filter}
\hat d_h(k+j|k)=\bm\varphi(\hat\psi_k+j\hat\Omega_kT_s)^{\mathrm T}\hat{\bm\vartheta}_k.
\end{equation}

\begin{lemma}[Prediction error]\label{lem:pred}
Let $\tilde{\bm\vartheta}_k=\bm\vartheta^\star_{k,0}-\hat{\bm\vartheta}_k$, $C_h=\sqrt{n_h}A_{\max}$ and $c_n=(\sum_{l=1}^{n_h}l^2)^{1/2}$. If the prediction interval is smooth as in Assumption~\ref{ass:A1}, then
\begin{equation}\label{eq:eps_h}
\begin{gathered}
|d_h(t_{k+j})-\hat d_h(k+j|k)|\le\varepsilon_h(k,j),\\
\begin{aligned}
\varepsilon_h(k,j)&:=\min\{B_h(k,j),\,2C_h\},\\
B_h(k,j)&=\sqrt{n_h}\big(\|\tilde{\bm\vartheta}_k\|+jT_s\delta_b\big)\\
&\quad+c_nA_{\max}\big(jT_s|\tilde\Omega_k|+\tfrac12(jT_s)^2\delta_\Omega\big).
\end{aligned}
\end{gathered}
\end{equation}
Across a frequency or coefficient switch, $\varepsilon_h(k,j)=2C_h$ remains valid. In particular, the error is uniformly bounded independently of convergence; the expression $B_h(k,j)$ need not itself be below $2C_h$.
\end{lemma}
\begin{proof}
Add and subtract the prediction using $\bm\vartheta^\star_{k,0}$ and the true phase. The coefficient and variation terms are bounded by $\sqrt{n_h}(\|\tilde{\bm\vartheta}_k\|+jT_s\delta_b)$ because $\|\bm\varphi\|=\sqrt{n_h}$. The phase discrepancy is at most $jT_s|\tilde\Omega_k|+(jT_s)^2\delta_\Omega/2$; Cauchy--Schwarz gives $c_nA_{\max}$ times this bound. Independently, $|d_h|\le C_h$ and $|\hat d_h|\le C_h$, so the error is at most $2C_h$. Taking the minimum proves \eqref{eq:eps_h}; only the amplitude argument is needed across switches. For details see Supplement~\ref{S-S05}.
\end{proof}
The coefficient error includes the effect of fitting the posterior signal in the presence of residual disturbances; projection bounds its magnitude without requiring exact separation or RLS convergence. Frequency and coefficient convergence~\cite{ioannouRobustAdaptiveControl1996} can tighten the local bound, whereas the amplitude bound alone supplies the uniform bound used below.

\textit{Drift weighting.} Recent movement of the frequency estimate is used as a proxy for unreliable extrapolation. Define $\Delta\hat f_k=|\hat\Omega_k-\hat\Omega_{k-L}|/(2\pi)$ and
\begin{equation}\label{eq:apdo_rho}
\begin{aligned}
\hat d^{\rho}_h(k+j|k)&=\rho_{k,j}\hat d_h(k+j|k),\\
\rho_{k,j}&=\mathrm{sat}_{[0,1]}\big(1-\kappa\,2\pi\Delta\hat f_k(j+n_d)T_s\big),
\end{aligned}
\end{equation}
where $\kappa\ge0$ is the drift gain; $\kappa_m$ denotes the normalization factor in the frequency estimator. The unweighted prediction is recovered for $\kappa=0$ or a stationary frequency estimate. For every $\rho_{k,j}\in[0,1]$, the triangle inequality yields
\begin{equation}\label{eq:eps_h_rho}
\begin{aligned}
|d_h(t_{k+j})-\hat d^{\rho}_h(k+j|k)|&\le\varepsilon_h^\rho(k,j),\\
\varepsilon_h^\rho(k,j)&:=\rho_{k,j}\varepsilon_h(k,j)+(1-\rho_{k,j})C_h\\
&\le2C_h.
\end{aligned}
\end{equation}
This weighted bound is smaller than $\varepsilon_h(k,j)$ when $\rho_{k,j}<1$ and $\varepsilon_h(k,j)>C_h$. This is a comparison of bounds: attenuation can increase or decrease the actual pointwise error. The frequency-drift statistic does not certify the unknown phase error. Its role is to reduce the magnitude of far-horizon predictions during frequency-estimation transients; the performance effect is evaluated in Section~\ref{sec:results}. The reported simulations use $\kappa=4$ and $L=50$ (1\,s), selected on tuning seeds.

\subsection{LSTM-DO Design}\label{sec:lstm_do}

As established in Section~\ref{sec:disturbance_modeling}, the aperiodic disturbance $d_{r,i}$ contains both slowly varying components (gravity and gyroscopic torques) and rapidly varying components (nonlinear friction). To capture such mixed-rate dynamics, an LSTM network is constructed to learn $d_{r,i}$ as the aperiodic disturbance. The current-value feedforward branch is optional. A separate frozen network predicts the residual disturbance change over the input delay, as described under lead correction below.

As established in Section~\ref{sec:disturbance_modeling}, the aperiodic disturbance $d_{r,i}$ contains both slowly varying components (gravity and gyroscopic torques) and rapidly varying components (nonlinear friction). An LSTM network is constructed to learn $d_{r,i}$ as the aperiodic disturbance. The current-value feedforward branch is optional. A separate frozen network predicts the residual disturbance change over the input delay, as described under lead correction below.

The feature vector at the $k$-th time step stacks historical system states and control inputs:
\begingroup
\setlength{\abovedisplayskip}{3pt}
\setlength{\belowdisplayskip}{3pt}
\begin{equation}\label{eq:lstm_input}
	\bm{\psi}_k = [ \bm{\phi}_k^\mathrm{T}, \bm{\phi}_{k-1}^\mathrm{T}, \bm{\phi}_{k-2}^\mathrm{T}, \bm{\phi}_{k-3}^\mathrm{T} ]^\mathrm{T},
\end{equation}
\endgroup
where $\bm{\phi}_k=\big[\theta_{a,k},\ \dot{\theta}_{a,k},\ \theta_{e,k},\ \dot{\theta}_{e,k}, \tau_{a,k-1},\ \tau_{e,k-1},\ \sin\theta_{a,k},\allowbreak\!\sin\theta_{e,k},\cos\theta_{a,k},\ \cos\theta_{e,k}\big]^\mathrm{T}\in\mathbb{R}^{10}$. Here $\theta_{i,k}$ and $\dot{\theta}_{i,k}$ represent the joint angle and angular velocity, $\tau_{i,k-1}$ is the torque applied at the previous step, and sine/cosine components capture gravity-relative torque dynamics.

For the current-value feedforward branch, offline training seeks a parameter set $\Theta$ such that the network output $\bm{\hat{d}}_{\mathrm{LSTM}}=[\hat{d}_{\mathrm{LSTM},a}, \hat{d}_{\mathrm{LSTM},e}]^\mathrm{T}$ approximates the two-axis aperiodic disturbance $\bm{d}_r=[d_{r,a},d_{r,e}]^\mathrm{T}$. Standard LSTM gating mechanisms and the parameter set $\Theta$ are detailed in Supplement~\ref{S-S03}. At the output layer, the spectrally normalized matrix $\bar{\mathbf{W}}_{\mathrm{out}}=\mathbf{W}_{\mathrm{out}}/\|\mathbf{W}_{\mathrm{out}}\|_2$ maps the hidden state to the disturbance estimate. The training target $\bm{d}_{r,k}^+$ is obtained from the posterior $d^+_i$ and the APDO estimates during offline construction.

The network is trained offline and kept frozen during online control. Because the hidden state satisfies $\|\mathbf h_k\|_\infty<1$ and spectral normalization ensures bounded raw output for any input sequence, each output channel is further projected in both magnitude and step increment:
\begin{equation}\label{eq:lstm_proj}
	\hat d_{L,k}=\Pi_{[\hat d_{L,k-1}-\delta_LT_s,\ \hat d_{L,k-1}+\delta_LT_s]}\Big(\mathrm{sat}_{d_L}\big(\hat d_{\mathrm{LSTM},k}\big)\Big),
\end{equation}
where $d_L>0$ and $\delta_L>0$ are the admissible magnitude and rate of the learned feedforward.

To account for disturbance variations not captured by offline training, a linear disturbance observer provides online correction around $\hat d_{L,k}$. In continuous time, the observer dynamics are $\dot{\hat\omega}=J^{-1}(\tau+H+\hat d_h+\hat d_L+z)+l_0(\omega-\hat\omega)$ and $\dot z=-l_2z+L_2(\omega-\hat\omega)$. Online implementation uses forward-Euler discretization with $n_s$ sub-steps of step size $h=T_s/n_s$, holding the sampled angular velocity $\omega_k$ constant across each sampling period:
\begin{equation}\label{eq:lstm_do_dynamics}
\left\{\begin{aligned}
\hat\omega^{+}&=\hat\omega+h\big[J^{-1}\big(\tau_{k-n_d}+H_k+\hat d_{h,k}+\hat d_{L,k}+z\big)\\
&\hspace{5em}+l_0(\omega_k-\hat\omega)\big],\\
z^{+}&=z+h\big[-l_2z+L_2(\omega_k-\hat\omega)\big],\\
\hat d_{r,k}&=\hat d_{L,k}+z_k,\\
\hat d^{\mathrm c}_{r,k}&=\mathrm{sat}_{\bar d}(\hat d_{r,k}+\Delta_k),
\end{aligned}\right.
\end{equation}
where $\hat\omega$ is the estimated angular velocity, $l_0$ is the velocity observer gain, $z$ is the online correction state, $l_2$ is its leakage rate, and $L_2$ is the disturbance correction gain. The control law applies the saturated estimate $\hat d^{\mathrm c}_{r,k}$ with bound $\bar d$, whereas the internal observer loop updates using the unsaturated estimate $\hat d_{r,k}$ without the lead term $\Delta_k$. The lead correction $\Delta_k$ is defined below, with $\Delta_k\equiv0$ recovering standard certainty-equivalence compensation.

\begin{lemma}[Discrete LSTM-DO error]\label{lem:do}
Let $e_{L,k}=d_{r,k}-\hat d_{L,k}$, $\tilde d_{r,k}=d_{r,k}-\hat d_{r,k}=e_{L,k}-z_k$, $\tilde\omega_k=\omega_k-\hat\omega_k$ and $\bm\zeta_k=[\tilde\omega_k,\tilde d_{r,k}]^{\mathrm T}$. Under Assumption~\ref{ass:A2},
\begin{equation}\label{eq:error_system}
\begin{aligned}
\bm\zeta_{k+1}&=\bm A_d\bm\zeta_k+\bm\eta_k,\\
\bm A_d&=(\bm I+h\bm A_c)^{n_s},\\
\bm A_c&=\begin{bmatrix}-l_0&J^{-1}\\-L_2&-l_2\end{bmatrix},
\end{aligned}
\end{equation}
\begin{equation}\label{eq:eta_bound}
\begin{aligned}
|\eta_{1,k}|&\le c_1T_s\,\bar\varepsilon_{h,0}^{\mathrm{obs}}+c_2T_s^2+\bar\eta_1^{\mathrm{mod}},\\
|\eta_{2,k}|&\le c_3\big(|e_{L,k+1}-e_{L,k}|+l_2T_s|e_{L,k}|\big)\\
&\quad+c_hT_s^2\bar\varepsilon_{h,0}^{\mathrm{obs}}+c_4T_s^2+\bar\eta_2^{\mathrm{mod}},
\end{aligned}
\end{equation}
where $\bar\varepsilon_{h,0}^{\mathrm{obs}}$ bounds the error of the current periodic estimate actually supplied to the observer (raw in the flight implementation, weighted in the simulation), and $c_1,\dots,c_4,c_h$ depend on $J,l_0,l_2,L_2,n_s$ and on the stated smoothness bounds on the envelope. The $c_h$ term accounts for periodic-estimate forcing propagated into the disturbance-error component across sub-steps. The additional $\bar\eta_i^{\mathrm{mod}}$ bound model and measurement discrepancies in the sampled observer recursion; they vanish for the exact nominal model. At discontinuities, the corresponding finite sampled remainder is included in this forcing budget. If $\bm A_c$ is Hurwitz ($l_0+l_2>0$, $l_0l_2+J^{-1}L_2>0$) and $h$ is such that $\bm I+h\bm A_c$ is Schur, then $\bm A_d$ is Schur and, with the element-wise gains $\gamma_{pq}=\sum_{l\ge0}\big|[\bm A_d^l]_{pq}\big|<\infty$,
\begin{equation}\label{eq:lstm_do_bound}
	\limsup_{k\to\infty}|\tilde d_{r,k}|\le\varepsilon_d:=\gamma_{21}\sup_k|\eta_{1,k}|+\gamma_{22}\sup_k|\eta_{2,k}|.
\end{equation}
\end{lemma}

\begin{proof}
Over a single sub-step, subtracting \eqref{eq:lstm_do_dynamics} from the exact plant increment and substituting $d-\hat d_{h,k}-\hat d_{r}=\tilde d_r+(d_h-\hat d_{h,k})$ yields the error system \eqref{eq:error_system} with forcing $\bm\eta_k$ and the stated bounds \eqref{eq:eta_bound}. The detailed sub-step expansion is in Supplement~\ref{S-S04}. Because $\bm A_d$ is Schur, $\bm\zeta_k=\bm A_d^k\bm\zeta_0+\sum_{l<k}\bm A_d^{k-1-l}\bm\eta_l$, and taking the second component via the triangle inequality gives \eqref{eq:lstm_do_bound}. For finite times, the bound includes $|[\bm A_d^{k-k_0}\bm\zeta_{k_0}]_2|$ in addition to $\varepsilon_d$. Element-wise gains are used rather than a single induced matrix norm because $\tilde\omega$ and $\tilde d_r$ have different physical units.
\end{proof}

With the bandwidth parameterization~\cite{gaoScalingBandwidthParameterization2003} $l_0=2\omega_o$, $l_2=0$ and $L_2=J\omega_o^2$, $\bm A_c$ has the double eigenvalue $-\omega_o$ and the discrete Schur condition is $0<h\omega_o<2$. Let $\delta_e$ bound the sampled rate of $e_L$. By \eqref{eq:lstm_proj}, $|e_{L,k+1}-e_{L,k}|\le T_s\delta_e\le T_s(\delta_r+\delta_L)$. The resulting bandwidth scaling and numerical element-wise gain checks are given in Supplement~\ref{S-S07}; the approximation there is restricted to the small-step regime.
\begin{remark}[Role of learning]\label{rem:learning}
Lemma~\ref{lem:do} holds for any frozen network, because only the architectural bound and the projections \eqref{eq:lstm_proj} are used; no time derivative of the network output is required. Thus, the observer-error bound holds for any frozen network satisfying these bounds. Training tightens the increment-dependent term when it reduces $\delta_e$; pointwise estimation accuracy alone does not imply a smaller increment bound. With $\hat d_L\equiv0$, $\delta_e=\delta_r$ and the scheme reduces to the linear DO. The rate projection gives $\delta_e\le\delta_r+\delta_L$, capping this learned-signal contribution to the observer forcing. The periodic and model-forcing budgets remain in \eqref{eq:lstm_do_bound}; the tracking effect is evaluated in Section~\ref{sec:res_learning}.
\end{remark}

\textit{Lead correction.} Because the command computed at step $k$ reaches the actuator at step $k+n_d$, the disturbance acting at execution time is $d_{r,k+n_d}$ rather than $d_{r,k}$. The delay-induced increment $D^\star_k:=d_{r,k+n_d}-d_{r,k}$ is unmeasured online and is predicted as the sum of a linear term and a learned residual. First, the model-based linear term extrapolates the filtered observer increment as $\Delta^{\mathrm{lin}}_k=n_d\,\chi_k$, where $\chi_k=(1-a)\chi_{k-1}+a(\hat d_{r,k}-\hat d_{r,k-1})$ with $a\in(0,1]$. Second, a frozen LSTM sharing the feature vector \eqref{eq:lstm_input} is trained offline to predict the residual $D^\star_k-\Delta^{\mathrm{lin}}_k$. Projecting its output via \eqref{eq:lstm_proj} with magnitude bound $\delta_D$ and rate bound $\delta_D'$ yields the learned lead residual $\hat\epsilon_k$ satisfying $|\hat\epsilon_k|\le\delta_D$. The combined lead correction is
\begin{equation}\label{eq:lead}
	\Delta_k=\mathrm{sat}_{D_{\max}}\big(\Delta^{\mathrm{lin}}_k+\hat\epsilon_k\big).
\end{equation}
\begin{lemma}[Lead error]\label{lem:lead}
Let $\varepsilon_D:=\sup_k|D^\star_k-\Delta_k|$, $M_r=n_dT_s\delta_r$, and $E_{\mathrm{lin}}=\sup_k|D^\star_k-\Delta^{\mathrm{lin}}_k|$. Under Assumption~\ref{ass:A2}, for any frozen network,
\begin{equation}\label{eq:lead_bound}
\varepsilon_D\le
\begin{cases}
\min\{M_r+D_{\max},E_{\mathrm{lin}}+\delta_D\},&D_{\max}\ge M_r,\\
M_r+D_{\max},&D_{\max}<M_r.
\end{cases}
\end{equation}
Without lead correction ($\Delta\equiv0$), $\varepsilon_D\le n_dT_s\delta_r$.
\end{lemma}
\begin{proof}
$|D^\star_k|\le n_dT_s\delta_r$ by Assumption~\ref{ass:A2} and $|\Delta_k|\le D_{\max}$, which gives the first term. If $D_{\max}\ge n_dT_s\delta_r$, then $D^\star_k\in[-D_{\max},D_{\max}]$ and the saturation in \eqref{eq:lead} is the projection onto an interval containing $D^\star_k$, which is non-expansive: $|D^\star_k-\Delta_k|\le|D^\star_k-\Delta^{\mathrm{lin}}_k-\hat\epsilon_k|\le|D^\star_k-\Delta^{\mathrm{lin}}_k|+\delta_D$.
\end{proof}
Because $\hat d_{r,k}$ is bounded by Lemma~\ref{lem:do} and $\chi_k$ is generated by a stable filter of its increments, $\sup_k|D^\star_k-\Delta^{\mathrm{lin}}_k|$ is finite and independent of the network. As in Remark~\ref{rem:learning}, the learned residual enters the worst-case bound only through $\delta_D$. An untrained or mismatched network can increase the bound by at most $\delta_D$, whereas accurate training reduces $\varepsilon_D$ below $\sup_k|D^\star_k-\Delta^{\mathrm{lin}}_k|$ only on data for which the network is accurate. Moreover, because $\Delta_k$ enters neither the observer recursion nor the MPC prediction model, it leaves Lemma~\ref{lem:do} and Proposition~\ref{prop:feas} unchanged.

The APDO and LSTM-DO form the disturbance separator. Its prediction--correction organization is inspired by biological internal-model and reflex coordination~\cite{kawatoInternalModelsMotor1999a}.

\subsection{Model Predictive Control (MPC)}\label{sec:mpc}

The nominal-history state comprises $\theta_k$, $\omega_k$, and $v_{k-1},\ldots,v_{k-n_d}$. For a chosen linear nominal model, periodic predictions enter as
\begin{equation}\label{eq:mpc_model}
\begin{aligned}
\bm z_{k+j+1|k}&=\bm A\bm z_{k+j|k}+\bm Bv_{k+j|k}+\bm E\hat d_h^\rho(k+j|k),\\
\bm z_{k|k}&=\bm z_k.
\end{aligned}
\end{equation}
Here $\bm E$ acts on the current physical state, whereas $\bm B$ inserts a newly computed command into the nominal-input delay queue. For $n_d>0$,
\[
\bm B=\operatorname{col}(\bm0_2,\bm e_1),\qquad
\bm E=\operatorname{col}(b_\theta,b_\omega,\bm0_{n_d}),
\]
where $\bm e_1$ is the first unit vector of $\mathbb R^{n_d}$ and $[b_\theta,b_\omega]^{\mathrm T}$ is the input column of the two-state zero-delay model. A current disturbance residual therefore enters through $\bm E$; $\bm E=\bm B$ only for $n_d=0$. The aperiodic estimate is compensated outside this prediction model. The torque-model realization cancels $H$ as in \eqref{eq:total_control}. The speed-interface simulation instead retains the identified first-order drift in $\bm A$ and uses $u=v-\hat d_r^{\mathrm c}$; the flight realization cancels $H=-\omega/K$ and uses the corresponding double-integrator nominal model. These implementation conventions and their residuals are specified in Sections~\ref{sec:sim_setup} and~\ref{sec:flight}.

To reserve margin $m\ge0$, impose $|\tau_k+H_k|\le u_{\max}-m$. With \eqref{eq:total_control}, this gives
\begin{equation}\label{eq:Vk}
\begin{aligned}
\mathcal V_k=\big[&\max(-u_{\max},-(u_{\max}-m)+\hat d^{\mathrm c}_{r,k}),\\
&\min(u_{\max},(u_{\max}-m)+\hat d^{\mathrm c}_{r,k})\big].
\end{aligned}
\end{equation}
When the compensated command $u_k=v_k-H_k-\hat d^{\mathrm c}_{r,k}$ has a separate bound $U_{\mathrm a}$, as in flight, the actual hard interval is
\begin{equation}\label{eq:actual_input_box}
\mathcal W_k=\mathcal V_k\cap[H_k+\hat d^{\mathrm c}_{r,k}-U_{\mathrm a},\,
H_k+\hat d^{\mathrm c}_{r,k}+U_{\mathrm a}].
\end{equation}
Write $\mathcal I_k=\mathcal V_k$ for the retained-drift simulation and $\mathcal I_k=\mathcal W_k$ for the compensated realization. All $N_c$ decision components use the interval evaluated at the current sample.

For $\bm V_k=[v_{k|k},\ldots,v_{k+N_c-1|k}]^{\mathrm T}$, define the quadratic tracking cost
\begin{equation}\label{eq:base_cost}
\begin{aligned}
\mathcal J_k(\bm V_k)&=\tfrac12\sum_{j=1}^{N_p}w_j\big[q_\theta(\theta_{k+j|k}-\theta_{\mathrm{ref},k+j})^2\\
&\qquad+q_\omega(\omega_{k+j|k}-\omega_{\mathrm{ref},k+j})^2\big]
+\tfrac r2\|\bm V_k\|^2,
\end{aligned}
\end{equation}
where $q_\theta>0$, $q_\omega\ge0$, $r>0$, $w_j=1$ for $j<N_p$ and $w_{N_p}=1+p$, $p\ge0$. The input is held at $v_{k+N_c-1|k}$ after the control horizon. The terminal multiplier is a scalar stage weight, with no terminal set or Riccati terminal cost.

For the state-soft realization, use relative-angle coordinates and define $e_{\theta,k+j|k}=\theta_{k+j|k}-\theta_{\mathrm{ref},k+j}$, with
\begin{equation}\label{eq:soft_rows}
\begin{aligned}
\bm y_{k,j}&=[\theta_{k+j|k},\omega_{k+j|k},e_{\theta,k+j|k}]^{\mathrm T},\\
\bm\ell&=[\theta_{\min},-\omega_{\max},-\theta_{\mathrm{FOV}}/2]^{\mathrm T},\\
\bm h&=[\theta_{\max},\omega_{\max},\theta_{\mathrm{FOV}}/2]^{\mathrm T}.
\end{aligned}
\end{equation}
The implemented problem is
\begin{equation}\label{eq:mpc_formulation}
\begin{aligned}
\min_{\bm V_k,\,\bm s_{k,j}\ge0}\quad&\mathcal J_k(\bm V_k)+\lambda_s\sum_{j=1}^{N_p}\bm1^{\mathrm T}\bm s_{k,j}\\
\text{s.t.}\quad&\eqref{eq:mpc_model},\quad\bm V_k\in\mathcal I_k^{N_c},\\
&\bm\ell-\bm s_{k,j}\le\bm y_{k,j}\le\bm h+\bm s_{k,j},\\
&v_{k+j|k}=v_{k+N_c-1|k}\quad(j\ge N_c).
\end{aligned}
\end{equation}
Each two-sided row shares one nonnegative slack. The flight implementation uses $3N_p$ slacks and $\lambda_s=1000$ with the one-half normalization in \eqref{eq:base_cost}. The box-only simulation omits the state rows and their penalty. Soft limits penalize predicted violations; they do not guarantee hard mechanical-range or optical-FOV satisfaction. The usual exact-penalty property~\cite{kerriganSoftConstraintsExact2000} requires an appropriate multiplier bound; no universal threshold is asserted for the chosen finite weight.

Writing $\mathcal J_k=\frac12\|\bm M\bm V_k-\bm b_k\|^2$ gives $\bm M^{\mathrm T}\bm M\succeq r\bm I\succ0$, including $q_\omega=0$. The unique unconstrained quadratic minimizer is
\begin{equation}\label{eq:full_sequence_law}
\bm V_k^0=-\bm F_z\bm z_k-\bm F_d\hat{\bm d}_{h,k}+\bm F_r\bm r_k,
\end{equation}
where $\hat{\bm d}_{h,k}=[\hat d_h^\rho(k|k),\ldots,\hat d_h^\rho(k+N_p-1|k)]^{\mathrm T}$. If $\bm b_k=\bm R\bm r_k-\bm P\bm z_k-\bm D\hat{\bm d}_{h,k}$, then $[\bm F_z,\bm F_d,\bm F_r]=(\bm M^{\mathrm T}\bm M)^{-1}\bm M^{\mathrm T}[\bm P,\bm D,\bm R]$. The first rows are $\bm K_z,\bm K_d,\bm K_r$. This affine law is also the constrained optimizer whenever the \emph{complete} sequence $\bm V_k^0$ lies in the hard box and its predicted state rows satisfy the unsoftened bounds: it then minimizes the quadratic part and has zero nonnegative slack cost, for any $\lambda_s>0$.

\begin{remark}
Only periodic predictions enter $\bm b_k$. Where the optimizer coincides with \eqref{eq:full_sequence_law}, the aperiodic compensation enters the plant as an additive residual; holding the aperiodic estimate constant inside the horizon instead introduces an additional optimizer feedforward map. Active bounds or state penalties can modify this mapping, which motivates the explicit margin conditions below.
\end{remark}

\subsection{Stability Analysis}\label{sec:stability}

The input-feasibility result applies throughout the stated operating envelope. The practical tracking result~\cite{jiangInputtostateStabilityDiscretetime2001,limonInputtoStateStabilityUnifying2009} additionally requires input-sequence and predicted-state margins, so that the optimizer coincides with \eqref{eq:full_sequence_law}.

\begin{assumption}\label{ass:A1}
The retained periodic component has a common fundamental $0<\Omega_{\min}\le\Omega\le\Omega_{\max}$, with $\Omega_{\max}T_sD<\pi$, and coefficient norm at most $A_{\max}$. On smooth prediction intervals, $|\dot\Omega|\le\delta_\Omega$ and the coefficient-variation bound $\delta_b$ of Lemma~\ref{lem:pred} holds. Predictions across switches use its uniform amplitude bound.
\end{assumption}
\begin{assumption}\label{ass:A2}
On the operating envelope, $|d_r|\le\bar d_r\le\bar d$, $|d_{r,k+1}-d_{r,k}|\le T_s\delta_r$, and the sampled observer forcing in \eqref{eq:eta_bound} is bounded. Derivative bounds used to obtain that forcing hold on smooth intervals; at jumps a finite sampled-increment bound is used. The nominal sampled plant has an additional bounded additive residual, accounting for model mismatch, intersample effects and implementation errors.
\end{assumption}
\begin{assumption}\label{ass:A3}
$\bar d<u_{\max}-m$. For the compensated realization using \eqref{eq:actual_input_box}, the operating envelope also satisfies $|H_k|\le\bar H$ and $\bar H+\bar d<U_{\mathrm a}$.
\end{assumption}
\begin{assumption}\label{ass:A4}
For the selected nominal model and finite-horizon weights, $\bm A_{\mathrm{cl}}=\bm A-\bm B\bm K_z$ is Schur, with $\bm K_z$ obtained from \eqref{eq:full_sequence_law}.
\end{assumption}
Positive $r$ ensures strict convexity, whereas Assumption~\ref{ass:A4} is a separate numerical or analytic check. The bounds on the operating envelope are assumptions, not consequences of the soft penalties.

\begin{proposition}[Input feasibility and command bounds]\label{prop:feas}
Under Assumption~\ref{ass:A3}, $[-b_0,b_0]\subseteq\mathcal I_k$, where $b_0=u_{\max}-m-\bar d>0$ for $\mathcal V_k$ and $b_0=\min\{u_{\max}-m-\bar d,U_{\mathrm a}-\bar H-\bar d\}>0$ for $\mathcal W_k$. Thus the QP is feasible at every step on the envelope. Its first decision satisfies $|v_k-\hat d^{\mathrm c}_{r,k}|\le u_{\max}-m$; with $\mathcal W_k$, the compensated command also satisfies $|u_k|\le U_{\mathrm a}$.
\end{proposition}
\begin{proof}
The bounded estimate and Assumption~\ref{ass:A3} place the indicated symmetric interval inside every interval in \eqref{eq:Vk}--\eqref{eq:actual_input_box}. Choose $\bm V_k=\bm0$ and slacks large enough to cover the finite predicted state rows. This proves feasibility independently of the reference and of the previous decision; membership of the first component in $\mathcal I_k$ gives the command bounds. Soft state rows require no additional feasibility condition.
\end{proof}
For an implementation outside the additional $H$-headroom envelope, nonemptiness of $\mathcal W_k$ must be checked online. The flight code stops the trial on an empty interval or an uncertified solve, rather than asserting feasibility there. Its integer command bound is checked after rounding.

For the tracking analysis, let $\bm z_k^\star$ be the ideal-information loop that uses \eqref{eq:full_sequence_law} with the exact periodic sequence $\bm d_{h,k}$, exact aperiodic compensation and zero model residual. Its tracking error is determined by the nominal controller and reference. Choose $k_0\ge n_d$ and fixed bounds $\bar\varepsilon_j\ge\sup_{k\ge k_0}\varepsilon_h^\rho(k,j)$, which exist by \eqref{eq:eps_h_rho}, and a finite-time observer bound
$\bar\varepsilon_d=\sup_{l\ge0}|[\bm A_d^l\bm\zeta_{k_-}]_2|+\varepsilon_d$, with $k_-=k_0-n_d$ so that the initial delayed estimates are covered. Define $\varepsilon_h=\sum_j|K_{d,j}|\bar\varepsilon_j$ and $\varepsilon_r=\bar\varepsilon_d+\varepsilon_D$. Projection of the aperiodic compensation gives $|d_{r,k}-\hat d^{\mathrm c}_{r,k-n_d}|\le\varepsilon_r$ by Lemma~\ref{lem:lead}. Collect the remaining bounded nominal-model and implementation residual in $\bm\nu_k$, with $\|\bm\nu_k\|\le\bar\nu$. In particular, delayed $H$ cancellation and command rounding are included in this residual, rather than identified with a measured external torque.

Let $\bm y_k^0$ stack all unsoftened predicted rows generated by \eqref{eq:full_sequence_law}, and let $\bm y_k^\star$ be their ideal-loop values. Prediction condensation gives
\begin{equation}\label{eq:sequence_sensitivity}
\begin{aligned}
\bm V_k^0-\bm V_k^\star&=-\bm F_z\delta\bm z_k-\bm F_d\delta\bm d_k,\\
\bm y_k^0-\bm y_k^\star&=\bm L_z\delta\bm z_k+\bm L_d\delta\bm d_k,
\end{aligned}
\end{equation}
where $\delta\bm d_k=\hat{\bm d}_{h,k}-\bm d_{h,k}$, $\bm L_z=\bm C(\bm\Phi-\bm\Gamma\bm F_z)$ and $\bm L_d=\bm C(\bm\Psi-\bm\Gamma\bm F_d)$; $\bm\Phi,\bm\Gamma,\bm\Psi$ are the stacked prediction matrices, and $\bm C$ selects the mechanical, speed and tracking-error rows. Reference offsets cancel in the difference. For the box-only controller, the second row and the corresponding state conditions below are omitted.

\begin{theorem}[Regional practical tracking bound]\label{thm:main}
Let Assumptions~\ref{ass:A1}--\ref{ass:A4} and the observer-gain conditions of Lemma~\ref{lem:do} hold. Put $\bm{\bar\varepsilon}=[\bar\varepsilon_0,\ldots,\bar\varepsilon_{N_p-1}]^{\mathrm T}$ and
\[
\Gamma_B=\sum_{l\ge0}\|\bm A_{\mathrm{cl}}^l\bm B\|,\quad
\Gamma_E=\sum_{l\ge0}\|\bm A_{\mathrm{cl}}^l\bm E\|,\quad
\Gamma_I=\sum_{l\ge0}\|\bm A_{\mathrm{cl}}^l\|.
\]
Choose $c\ge1$ and $\rho_c\in(0,1)$ with $\|\bm A_{\mathrm{cl}}^l\|\le c\rho_c^l$, and set
$\mathcal B=c\|\delta\bm z_{k_0}\|+\Gamma_B\varepsilon_h+\Gamma_E\varepsilon_r+\Gamma_I\bar\nu$.
Suppose the entire ideal input sequence satisfies $\|\bm V_k^\star\|_\infty\le b_0-\mu_V$ for some $\mu_V>0$, and every ideal predicted state row has distance at least $\mu_i>0$ from both its bounds, for all $k\ge k_0$. If
\begin{equation}\label{eq:region}
\begin{aligned}
\max_i\|(\bm F_z)_{i,:}\|\,\mathcal B+
\big\||\bm F_d|\bm{\bar\varepsilon}\big\|_\infty&<\mu_V,\\
\|(\bm L_z)_{i,:}\|\,\mathcal B+
\big[|\bm L_d|\bm{\bar\varepsilon}\big]_i&<\mu_i\quad\text{for every state row }i,
\end{aligned}
\end{equation}
and the deviation tube of radius $\mathcal B$ stays in the operating envelope, then the exact optimizer agrees with \eqref{eq:full_sequence_law} and its soft slacks are zero. The deviation obeys
\begin{equation}\label{eq:iss}
\begin{aligned}
\|\bm z_k-\bm z_k^\star\|&\le c\rho_c^{k-k_0}\|\delta\bm z_{k_0}\|\\
&\quad+\Gamma_B\varepsilon_h+\Gamma_E\varepsilon_r+\Gamma_I\bar\nu,
\qquad k\ge k_0.
\end{aligned}
\end{equation}
This is a regional input-to-state practical bound for tracking deviation relative to the ideal-information loop, conditional on the complete input-sequence and predicted-state margins. It holds for any frozen projected network for which these error and margin conditions hold.
\end{theorem}
\begin{proof}
The Schur property gives finite convolution gains. The periodic prediction bound and finite-time delayed-observer bound yield the forcing budgets $\varepsilon_h$ and $\varepsilon_r$. Starting at $k_0$, \eqref{eq:sequence_sensitivity} and \eqref{eq:region} place the complete quadratic minimizer inside the input box and every unsoftened predicted-state bound. It therefore minimizes the full objective with zero slacks. Subtracting the two affine-loop recursions and applying the Schur convolution bound gives \eqref{eq:iss}. This bound preserves the budget $\mathcal B$; the same row margins and envelope inclusion close the induction. The full delayed-error and convolution argument is given in Supplement~\ref{S-S08}.
\end{proof}

\begin{remark}
The state rows in \eqref{eq:region} and the \emph{whole} $N_c$-dimensional input-sequence margin are essential: a later decision can activate a hard bound and change the first decision, even when that first decision has headroom. Active state penalties can also change the affine law. Outside this region, Proposition~\ref{prop:feas} gives input feasibility on its envelope, but no stability conclusion is inferred for active-constraint motion. The drift weight preserves boundedness and changes the prediction-error budget; its accuracy benefit is empirical. Learning can reduce the aperiodic error budget when its predictions are accurate, while its projected magnitude and rate keep that budget finite. Numerical optimization errors can be included in $\bm\nu_k$; the theorem does not identify a small solver residual with an exact physical model.
\end{remark}

\begin{remark}[Numerical check of the nominal feedback]\label{rem:numeric}
The selected nominal feedbacks satisfy the Schur check: their yaw/pitch spectral radii are $0.6202/0.5979$ for the retained-drift simulation model and $0.7788/0.7807$ for the separately identified flight model. Supplement~\ref{S-S09} gives the weights and impulse gains. These checks address Assumption~\ref{ass:A4}; they do not certify the trajectory-dependent margins of \eqref{eq:region} for the flights.
\end{remark}

\section{Simulation and Experimental Results}\label{sec:results}

\subsection{Simulation Setup}\label{sec:sim_setup}

The scheme is evaluated on the yaw axis of a commercial two-axis gimbal driven through its speed interface. The plant, identified by frequency sweeps, is the input-equivalent model $\dot\omega=\big(-\omega+K(u(t-T_d)+d)\big)/\tau_m$ with $K=0.805$, $\tau_m=0.109$\,s and $T_d=0.059$\,s, where $u$ is the integer speed command in levels, $|u|\le u_{\max}=50$, so that $J=\tau_m/K$, $H=-\omega/K$ and $n_d=3$ at $T_s=0.02$\,s. For these simulations, the first-order drift is retained in the nominal state matrix, and the applied speed command is $u_k=v_k-\hat d^{\mathrm c}_{r,k}$; mechanical and tracking-error soft rows are omitted. This is the box-only realization of Section~\ref{sec:mpc}. The plant is integrated at 1\,kHz and includes the measured non-idealities: a dead zone of 3 levels, gain asymmetry, quantization, and command jitter. Two plant variants are used. \emph{Nominal} is the model above. \emph{Calibrated} adds an extra first-order lag of 22\,ms, angle/rate-readout noise and a rate-readout scale of $1/1.105$, reproducing measured RMSE within 0.02\,deg for most methods; it is the primary test case.

Disturbances with known profiles are injected directly into the command channel. The periodic component is $d_h=a_1\sin(2\pi ft+\phi_1)+a_2\sin(4\pi ft+\phi_2)$ with amplitudes $a_1\in[6.4,9.6]$ and $a_2\in[1.2,2.4]$ levels and frequency $f=f_0\,\mathcal U[0.9,1.1]$. In the \emph{steady} scenario, the frequency remains constant, whereas in the \emph{switch} scenario, it jumps to $1.6f$ at $t=20$\,s. The aperiodic component combines a random walk, a step, and the velocity-dependent friction term. The angular reference is $3\sin(0.8t)$\,deg; each run lasts 30\,s, and the evaluation metrics exclude the first 5\,s transient. Four base frequencies $f_0\in\{0.5,1,1.5,2\}$\,Hz span the band in which periodic vibrations degrade pointing accuracy.

Seven controllers share the MPC weights ($q_\theta=1$, $q_\omega=0$, $r=10^{-4}$, $N_p=20$, $N_c=3$) and the input limit: MPC, MPC without compensation; MPC-DO, MPC with a single linear DO~\cite{chenDisturbanceobserverbasedControlRelated2015}; MPC-APDO, MPC with the original APDO~\cite{yangAdaptivePeriodicdisturbanceObserver2023} and a DO, both compensated at their current value; MPC with fixed frequency, the nominal disturbance-preview structure of~\cite{Xie2026Automatica} with a full-order observer and a fixed internal-model frequency $f_0$; MPC with adaptive frequency, the same nominal structure with the internal-model frequency adapted online; MPC-APDO (preview), which predicts $d_h$ over the horizon as in Section~\ref{sec:apdo} but uses the original APDO frequency law and holds the DO estimate constant over the horizon; MPC-DS, the proposed scheme in Section~\ref{sec:control_design}, configured with $A_{\max}=12$, $m=5$, $\bar d=30$ levels, and the drift weight \eqref{eq:apdo_rho} with $\kappa=4$, $L=50$; MPC-DS ($\kappa=0$) denotes the same scheme without drift weighting. The learned feedforward and lead correction are disabled in the main simulation comparisons so that MPC-DS isolates the model-based prediction and compensation core. Parameter tuning details are in Supplement~\ref{S-S06}. All test evaluations use the disjoint seeds 100--109, providing 40 runs per method and scenario.
\subsection{Tracking under Periodic Disturbances}\label{sec:res_main}

Tracking accuracy, command saturation, and sensitivity to the frequency switch are compared across all controllers in Table~\ref{tab:res_main}. The relative change (Rel.) denotes the paired RMSE ratio relative to MPC-DO on the same seed, averaged across the seeds for each $f_0$ (using the median when a cell contains a high-error run) and then across the four base frequencies $f_0$. For this comparison, runs with RMSE $\ge1$\,deg are counted as high-error runs (column Fail.).

\begin{table*}[!htbp]
	\centering
	\caption{Test results over 40 runs per entry (10 seeds $\times$ 4 base frequencies). RMSE in deg; Rel.: paired change relative to MPC-DO (\%); Sat.: share of samples with $|u|\ge u_{\max}$ (\%); Fail.: count of high-error runs with RMSE $\ge1$\,deg.}
	\label{tab:res_main}
	\small
	\setlength{\tabcolsep}{3pt}
	\begin{tabular*}{\textwidth}{@{\extracolsep{\fill}}lcccccccc@{}}
		\toprule
		& \multicolumn{4}{c}{\textbf{Calibrated, steady}} & \multicolumn{4}{c}{\textbf{Calibrated, switch}} \\
		\cmidrule(lr){2-5}\cmidrule(lr){6-9}
		\textbf{Method} & RMSE & Rel. & Sat. & Fail. & RMSE & Rel. & Sat. & Fail. \\
		\midrule
		MPC    & 0.389 & $+24.0$  & 0.1  & 0  & 0.382 & $+27.9$  & 0.1  & 0 \\
		MPC-DO    & 0.320 & 0        & 0.1  & 0  & 0.303 & 0        & 0.1  & 0 \\
		MPC-APDO   & 0.295 & $-3.4$   & 1.2  & 0  & 0.333 & $+15.4$  & 0.9  & 0 \\
		\makecell[l]{MPC with\\fixed frequency}   & 0.934 & $+201$   & 90.4 & 17 & 0.916 & $+219$   & 90.9 & 11 \\
		\makecell[l]{MPC with\\adaptive frequency} & 0.898 & $+188$   & 88.7 & 7  & 0.874 & $+197$   & 88.3 & 3 \\
		MPC-APDO (preview)    & 0.210 & $-29.9$  & 0.9  & 0  & 0.281 & $-2.8$   & 0.9  & 0 \\
		MPC-DS ($\kappa=0$) & 0.138 & $-55.5$ & 0 & 0 & 0.203 & $-32.2$ & 0 & 0 \\
		\textbf{MPC-DS} & $\bm{0.129}$ & $\bm{-58.1}$ & \textbf{0} & \textbf{0} & $\bm{0.182}$ & $\bm{-39.1}$ & \textbf{0} & \textbf{0} \\
		\midrule
		& \multicolumn{4}{c}{\textbf{Nominal, steady}} & \multicolumn{4}{c}{\textbf{Nominal, switch}} \\
		\cmidrule(lr){2-5}\cmidrule(lr){6-9}
		MPC-DO    & 0.217 & 0        & 0    & 0  & 0.217 & 0        & 0    & 0 \\
		\makecell[l]{MPC with\\fixed frequency}   & 0.100 & $-52.3$  & 20.2 & 0  & 0.142 & $-36.2$  & 19.1 & 0 \\
		\makecell[l]{MPC with\\adaptive frequency} & 0.086 & $-58.2$  & 20.0 & 0  & $\bm{0.100}$ & $\bm{-53.2}$ & 19.0 & 0 \\
		MPC-APDO (preview)    & 0.115 & $-44.2$  & 5.4  & 0  & 0.156 & $-28.5$  & 5.0  & 0 \\
		MPC-DS ($\kappa=0$) & 0.089 & $-54.8$ & 0 & 0 & 0.139 & $-33.2$ & 0 & 0 \\
		\textbf{MPC-DS} & $\bm{0.080}$ & $\bm{-58.7}$ & \textbf{0} & \textbf{0} & 0.122 & $-41.5$ & \textbf{0} & \textbf{0} \\
		\bottomrule
	\end{tabular*}
\end{table*}

On the calibrated plant, MPC-DS reduces the RMSE relative to MPC-DO by 58.1\% at constant frequency and by 39.1\% with the frequency switch, with no saturated sample and no high-error run. Without the drift weight the reductions are 55.5\% and 32.2\%; the weight lowers the mean RMSE by 5.9\% and 10.6\% (better in 35 of 40 runs at constant frequency, Wilcoxon $p=3.5\times10^{-6}$, and in all 40 with the switch, $p=1.8\times10^{-12}$). Current-value compensation (MPC-APDO) cannot account for the lag of the inner loop, and its benefit turns into a loss above 1.5\,Hz. The two frequency-preview variants, which perform well on the nominal plant, saturate in about 90\% of the samples on the calibrated plant; Fig.~\ref{fig:box} shows this is a systematic failure. The 22\,ms unmodeled lag and the readout noise are not part of their observer model, and the observer error, held constant over the horizon, is amplified by the MPC. MPC-APDO (preview) shares this structure and is affected to a lesser degree.
\begin{figure*}[pos=t]
	\centering
	\includegraphics[width=\textwidth]{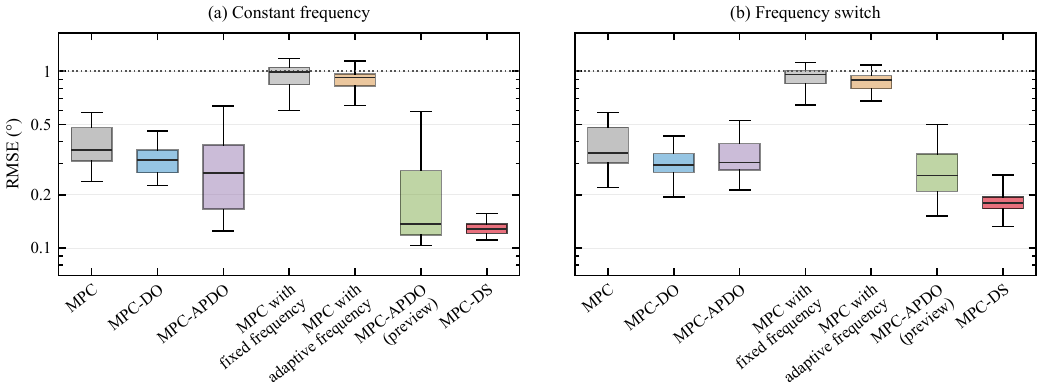}
	\caption{Distribution of the pointing RMSE over 40 test runs on the calibrated plant (boxes: quartiles; whiskers: minimum and maximum; dotted line: 1\,deg high-error threshold).}
	\label{fig:box}
\end{figure*}

On the nominal plant without unmodeled lag, MPC with adaptive frequency achieves the lowest RMSE in the switch scenario ($-53.2\%$ versus $-41.5\%$ for MPC-DS), while MPC with fixed frequency and MPC with adaptive frequency match the accuracy of MPC-DS at constant frequency ($-52.3\%$ and $-58.2\%$ versus $-58.7\%$), though with 19--20\% saturated samples. Across both plant variants, MPC-DS maintains competitive nominal tracking accuracy while avoiding command saturation and preserving accuracy when the plant deviates from the prediction model.

Fig.~\ref{fig:time} shows a representative switch run on the calibrated plant ($f_0=1$\,Hz, seed 100). Before the switch (10--20\,s), the RMSE is 0.159\,deg for MPC-DS against 0.375\,deg for MPC-DO. During the 2\,s after the switch, MPC-DS is worse (0.456 versus 0.319\,deg; 0.570\,deg without the drift weight), because the predictor extrapolates the old frequency until the estimate has moved. From 22 to 30\,s the two are again comparable (0.283 and 0.272\,deg). The command of MPC-DS stays within $\pm(u_{\max}-m)$ throughout, as guaranteed by Proposition~\ref{prop:feas}, whereas MPC with fixed frequency is saturated in 91\% of the samples.
\begin{figure}[pos=t]
	\centering
	\includegraphics[width=\columnwidth]{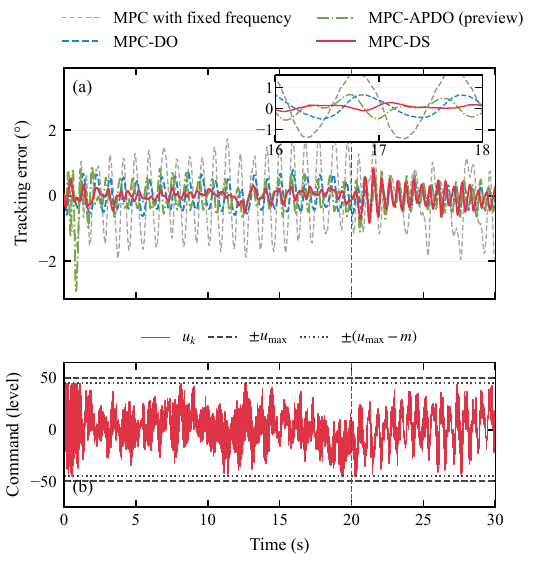}
	\caption{Representative run on the calibrated plant, $f_0=1$\,Hz with the switch to $1.6f$ at $t=20$\,s (dashed vertical line): (a) tracking error, with an inset showing 16--18\,s; (b) integer speed-command level of MPC-DS with the input limit and the tightened limit $u_{\max}-m$.}
	\label{fig:time}
\end{figure}

\subsection{Frequency Estimation and Ablation}\label{sec:res_ablation}

Each component of MPC-DS was removed in turn, with the parameters selected for MPC-DS and no retuning: w/o drift weighting is MPC-DS ($\kappa=0$); w/ gradient frequency estimation uses the gradient frequency law of~\cite{yangAdaptivePeriodicdisturbanceObserver2023} instead of \eqref{eq:apdo_freq_estimator}; w/o periodic prediction sets $\hat d_h\equiv0$; w/o input tightening sets $m=0$; w/ horizon-held aperiodic estimate feeds $\hat d_r$ into the horizon; w/o amplitude projection removes the amplitude projection. The results on the calibrated plant are shown in Fig.~\ref{fig:ablation}.

\begin{figure}[pos=t]
	\centering
	\includegraphics[width=\columnwidth]{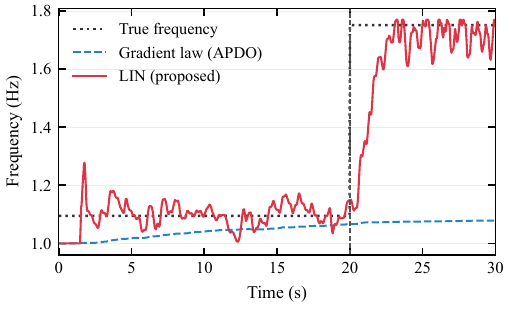}
	\caption{Frequency estimates in the run of Fig.~\ref{fig:time}: proposed linear-regression law \eqref{eq:apdo_freq_estimator} and the gradient law of~\cite{yangAdaptivePeriodicdisturbanceObserver2023}, both with the same prefilter input.}
	\label{fig:freq}
\end{figure}

\begin{figure*}[pos=t]
	\centering
	\includegraphics[width=\textwidth]{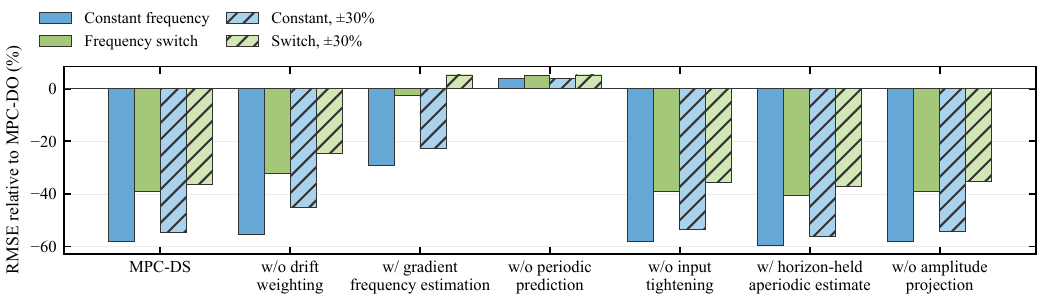}
	\caption{Ablation on the calibrated plant: paired RMSE change relative to MPC-DO, without mismatch and with $\pm30\%$ parameter mismatch (hatched).}
	\label{fig:ablation}
\end{figure*}

Removing periodic prediction reduces MPC-DS to MPC-DO ($+4.0\%$ and $+5.0\%$), confirming that the primary tracking improvement originates from predicting the periodic disturbance over the horizon. Among the predictor modules, the frequency estimation law has the largest impact: substituting the gradient law reduces the relative RMSE improvement from $-58.1\%$ to $-29.3\%$ at constant frequency and from $-39.1\%$ to $-2.6\%$ under the frequency switch. Fig.~\ref{fig:freq} shows that before the switch, both estimators track the true frequency closely (mean errors of 0.029 and 0.041\,Hz). After the switch, the gradient law fails to reach the new frequency within the remaining 10\,s (mean error of 0.674\,Hz over 25--30\,s), whereas the proposed linear-regression estimator enters a 5\% band after 9.8\,s (mean error of 0.040\,Hz over 25--30\,s). The drift weight is the second most influential component (w/o drift weighting: $-55.5\%$ and $-32.2\%$), and its share grows under mismatch (Section~\ref{sec:res_mismatch}). The remaining three components change the relative RMSE by at most 1.5 percentage points on the calibrated plant with and without $\pm30\%$ mismatch; their role shows in the guarantees and on the nominal plant. Without the margin the command saturates in 4--7\% of the samples at $\pm30\%$, against none for MPC-DS. With the aperiodic estimate held over the horizon, the nominal plant is more accurate without mismatch ($-65.0\%$ versus $-58.7\%$), but at $\pm30\%$ the gain falls to $-25.6\%$ and $-19.8\%$ against $-50.7\%$ and $-36.8\%$ for MPC-DS. The amplitude projection has no measurable effect on accuracy in any case tested; it is needed for the uniform bound of Lemma~\ref{lem:pred}.
\subsection{Robustness to Parameter Mismatch}\label{sec:res_mismatch}
Fig.~\ref{fig:mismatch} shows the mean RMSE on the calibrated plant as the mismatch bound increases from 0 to 30\%, with the parameters tuned for the identified plant. MPC-DS increases from 0.129 to 0.148\,deg at constant frequency and from 0.182 to 0.204\,deg with the switch ($-54.7\%$ and $-36.4\%$ relative to MPC-DO at $\pm30\%$), without any saturated sample, and remains below MPC-DO at every level. The drift weight is most useful here: without it the RMSE rises to 0.184 and 0.245\,deg at $\pm30\%$ (19.9\% and 16.9\% lower with the weight, better in 31 and 36 of 40 runs, $p\le10^{-5}$). With mismatch the frequency estimate fluctuates more, and the weight discounts the far horizon in these phases. MPC-APDO (preview) degrades from 0.210 to 0.312\,deg, saturates in 11--12\% of the samples and exceeds the high-error threshold in 3 runs at 30\%. On the nominal plant the picture is the same: at $\pm30\%$, MPC-DS keeps $-50.7\%$ and $-36.8\%$ relative to MPC-DO, whereas MPC with adaptive frequency and MPC-APDO (preview) fall behind MPC-DO and saturate in 38--58\% of the samples.

\begin{figure*}[pos=t]
	\centering
	\includegraphics[width=\textwidth]{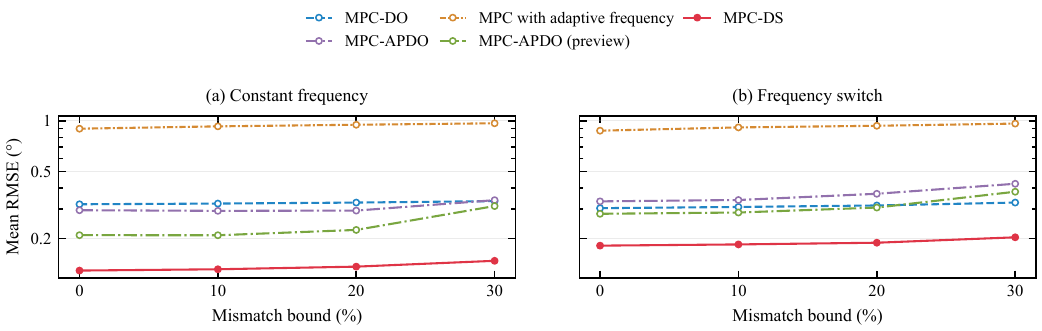}
	\caption{Mean RMSE on the calibrated plant versus the bound of the uniform mismatch in $K$, $\tau_m$ and $T_d$ (40 runs per point, parameters not retuned).}
	\label{fig:mismatch}
\end{figure*}

\begin{figure*}[pos=t]
    \centering
    \includegraphics[width=\textwidth]{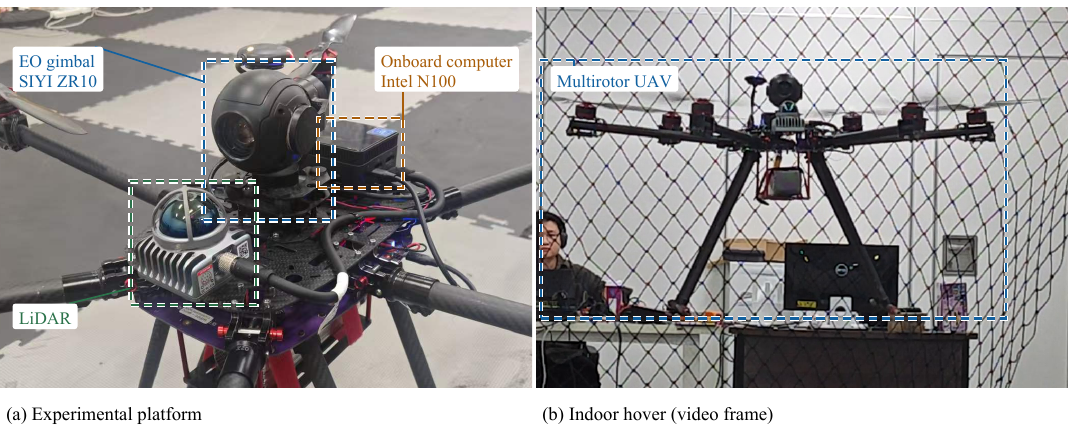}
    \caption{Experimental platform and indoor flight. (a) Hardware close-up; dashed boxes identify the SIYI ZR10 electro-optical gimbal, Intel N100 onboard computer, and LiDAR. (b) Video frame of the multirotor hovering inside a netted indoor test area.}
    \label{fig:platform}
\end{figure*}

\subsection{Learned Components}\label{sec:res_learning}
The learned current-value feedforward $\hat d_L$ and the learned lead residual $\hat\epsilon$ were added to MPC-DS one at a time, each network trained on closed-loop data of seeds 0--9 and frozen, and $\omega_o$ and the projection bounds were retuned for each variant. The learned feedforward gave no significant gain: $-0.9\%$ on the nominal plant and $-1.1\%$ on the calibrated plant relative to MPC-DS without learning (22 and 24 of 40 runs better), and $+3.9\%$ on the calibrated plant when trained on hardware logs. The lead correction \eqref{eq:lead} with the learned residual was effective: the RMSE decreased by 4.4\% on the nominal plant (33 of 40 runs better, Wilcoxon $p=2\times10^{-6}$) and by 2.2\% on the calibrated plant (28 of 40, $p=2\times10^{-3}$), the gain being largest at $f_0=2$\,Hz. The linear part alone changed the RMSE by $+1.0\%$ (nominal) and $-1.8\%$ (calibrated), and with the residual network switched off the nominal RMSE increased by 15.4\%, so the gain comes from the learned residual. These comparisons locate the useful learned correction in the delay-induced change $D^\star_k$: the current-value estimate is already supplied effectively by the linear DO at the selected bandwidth. The projection bounds in Remark~\ref{rem:learning} and Lemma~\ref{lem:lead} bound the learned signals; the observed RMSE reductions are performance evidence for the tested cases. The gains are smaller than with $\kappa=0$, consistent with overlapping performance effects of drift weighting and lead correction.
\subsection{Flight Experiment}\label{sec:flight}
To evaluate two-axis angular tracking in flight, indoor hovering experiments were conducted on a multirotor UAV equipped with a SIYI ZR10 electro-optical gimbal, as shown in Fig.~\ref{fig:platform}. During the tests, the UAV hovered indoors approximately 1\,m above the floor while an onboard Intel N100 computer executed the control laws at a nominal sampling period of 20\,ms. The gimbal was driven through its speed-command interface using the input-equivalent formulation in Section~\ref{sec:sim_setup}, with independently identified yaw and pitch models. Both axes tracked periodic angular references with an amplitude of $8^\circ$ a frequency of 0.6\,Hz, and a $90^\circ$ phase offset over 30\,s trials. The demonstration video, reproducibility code, processed figure data and plotting scripts are available at \url{https://github.com/Awilonk/airborne-gimbal-mpc-ds}.

\begin{figure*}[pos=t]
    \centering
    \includegraphics[width=\textwidth]{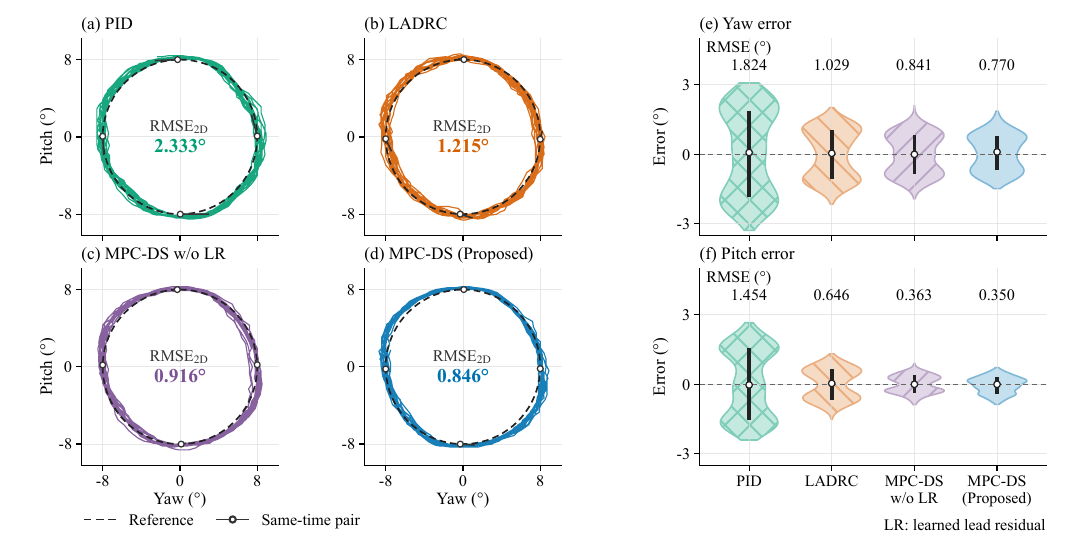}
    \caption{Two-axis angular tracking during indoor hovering. (a)--(d) Measured angular trajectories for PID, LADRC, MPC-DS w/o LR, and MPC-DS (Proposed) over 5--28\,s; the dashed circle is the reference, and black segments with open circles mark four simultaneous reference/measurement pairs in the same reference cycle. The circle-center values give the joint two-axis RMSE, defined from simultaneous yaw and pitch errors. (e), (f) Signed yaw and pitch error densities over the same interval, with colors matching (a)--(d); white dots and dark segments denote medians and interquartile ranges, and the numbers give RMSEs in degrees. Each method contributes one 30\,s trial; the densities describe time samples within that trial. PID is the plain-PID trial collected in a separate flight batch.}
    \label{fig:hover}
\end{figure*}

The flight adapter uses $u_k=v_k-H_k-\hat d^{\mathrm c}_{r,k}$ with $H_k=-\omega_k/K$. The yaw parameters are $K=1.2491$, $\tau_m=0.1957$\,s and $T_d=0.060$\,s; the pitch parameters are $K=1.4387$, $\tau_m=0.2433$\,s and $T_d=0.040$\,s. It uses $N_p=20$, $N_c=3$, $q_\theta=1$, $q_\omega=0$, $r=10^{-3}$ and $p=1$, with $u_{\max}=100$, $m=5$, $\bar d=30$ and $U_{\mathrm a}=90$ levels. The mechanical and speed soft bounds are $\pm10^\circ$ relative to the trial origin and $\pm60^\circ/\mathrm s$; the yaw/pitch tracking-error soft bounds are $\pm4^\circ/\pm3^\circ$, used as FOV proxies without an optical FOV calibration. The solver first checks the box solution; if soft rows are violated, it solves the explicit slack QP, requiring a certified objective gap at most $10^{-5}$. An empty input interval, failed solve or computation timeout stops the trial. The physical rate lag, the delayed cancellation residual $H_k-H_{k-n_d}$, rounding and timing variations are implementation effects represented by $\bm\nu_k$ in the analysis.

The proposed controller was compared with a PID controller with anti-windup and without model feedforward, a bandwidth-parameterized LADRC~\cite{gaoScalingBandwidthParameterization2003}, and MPC-DS w/o LR, where LR denotes the learned lead residual. In MPC-DS w/o LR, only $\hat\epsilon_k$ in \eqref{eq:lead} is disabled; the linear lead term and other parameters are retained. One trial per method is reported. The reported errors use the gimbal's relative angular feedback; no independent inertial line-of-sight calibration or artificial disturbance injection was used in these trials. The PID trial was collected in a separate flight batch. All RMSEs were computed from the original, unfiltered samples over the common 5--28\,s interval, excluding the initial transient and ending reference ramp. Figs.~\ref{fig:hover}(a)--(d) compare the two-axis angular trajectories, and Figs.~\ref{fig:hover}(e) and~\ref{fig:hover}(f) show the yaw and pitch error distributions with their RMSEs.

The joint two-axis metric is defined from simultaneous angular error pairs as
\begin{equation*}
\mathrm{RMSE}_{\mathrm{2D}}=\sqrt{\frac{1}{N}\sum_{k=1}^{N}(e_{\mathrm{yaw},k}^{2}+e_{\mathrm{pitch},k}^{2})}.
\end{equation*}
where the paired errors are measured in degrees. Across the evaluation window, PID, LADRC, MPC-DS w/o LR, and the proposed MPC-DS achieve joint two-axis RMSEs of $2.333^\circ$, $1.215^\circ$, $0.916^\circ$, and $0.846^\circ$, respectively. On the individual axes, the proposed controller attains yaw and pitch RMSEs of $0.770^\circ$ and $0.350^\circ$, corresponding to reductions of 25.1\% and 45.9\% relative to LADRC, and 8.4\% and 3.8\% relative to MPC-DS w/o LR.

Deterministic replay of the Proposed trial reproduced the commands to within $2\times10^{-11}$ levels. At all 1493 recorded control samples, both axis QPs succeeded with zero predicted-state slacks; the full yaw input sequence reached a hard bound at two samples, whereas the pitch sequence did not. These observations verify the recorded optimization behavior and do not certify that the entire trial satisfies the inactive-region conditions of Theorem~\ref{thm:main}. The hover tests demonstrate two-axis angular tracking on the UAV-mounted gimbal. The recorded comparison with MPC-DS w/o LR shows lower errors with the learned lead residual, particularly on the yaw axis. The controlled simulations evaluate the disturbance-separation and frequency-adaptation mechanisms.

\section{Conclusion}\label{sec:conclusion}

This paper presented a disturbance-separated model predictive control (MPC-DS) scheme for airborne electro-optical gimbal pointing under composite periodic and aperiodic disturbances and multiple constraints. An adaptive periodic disturbance observer with a projected linear-regression frequency estimator and drift weighting predicts harmonic vibrations over the MPC horizon, while a bounded LSTM-augmented observer compensates aperiodic disturbances and input-delay increments outside the optimizer. The input-set construction ensures feasible bounded commands on the stated headroom envelope. Under full-sequence input margins, predicted-state margins and bounded model residuals, the closed loop satisfies a regional practical tracking bound relative to its ideal-information nominal loop. On a gimbal model calibrated against hardware logs, the model-based core of MPC-DS reduced the pointing RMSE by 58\% at constant vibration frequency and by 39\% across a frequency switch relative to MPC with a linear disturbance observer, with zero command saturation and graceful retention under 30\% parameter mismatch. Ablation experiments showed that learning is effective specifically when predicting the delay-induced change of the aperiodic disturbance, yielding an additional 2--4\% RMSE reduction, whereas learning a current-value feedforward term provided no significant gain over the linear observer. In indoor hovering tests with a UAV-mounted two-axis gimbal, the proposed controller achieved yaw and pitch relative-angle tracking RMSEs of $0.770^\circ$ and $0.350^\circ$, reducing the RMSE by 25\% and 46\% relative to bandwidth-parameterized LADRC in single-trial evaluations. Future work will examine repeated flight trials and additional carrier motions.

\bibliographystyle{elsarticle-num}
\bibliography{references}

\input{supplement}
\end{document}

%% file: supplement.tex
\twocolumn[\centering{\Large\bfseries Supplementary material\par}\vspace{1em}]
\renewcommand{\thesection}{S\arabic{section}}
\renewcommand{\theequation}{S\arabic{equation}}
\renewcommand{\thefigure}{S\arabic{figure}}
\renewcommand{\thetable}{S\arabic{table}}
\renewcommand{\thelemma}{S\arabic{lemma}}
\renewcommand{\theremark}{S\arabic{remark}}
\renewcommand{\theHsection}{supp.\arabic{section}}
\setcounter{section}{0}
\setcounter{equation}{0}
\setcounter{figure}{0}
\setcounter{table}{0}
\setcounter{lemma}{0}
\setcounter{remark}{0}

\noindent This supplementary material provides detailed derivations, complete proofs and implementation checks supporting the main article. Equation numbers beginning with S refer to this supplementary material; other equation numbers refer to the main article.

\section{Detailed Rigid-Body Derivation}\label{S-S01}
This section presents the complete Newton--Euler derivation of the two-axis gimbal dynamics. The platform angular velocities satisfy the rigid-body kinematic relations:
The ISP is equipped with a gyroscope, motor encoders, and an electro-optical camera. The mechanism consists of an outer azimuth gimbal (body A) rotating about the $z$-axis by $\theta_a$ and an inner elevation gimbal (body E) rotating about the $y$-axis by $\theta_e$, with the camera LOS aligned with the elevation platform. Choosing $\theta_a$ and $\theta_e$ as the joint angles of the azimuth and elevation axes relative to the gimbal base and the azimuth frame, respectively, the platform angular velocities satisfy the rigid-body kinematic relations:
\begin{equation}\label{S-eq:kinematics}
	\left\{
	\begin{aligned}
		\bm{\omega}_a &= \bm{R}_B^a\bm{\omega}_B + [0\;0\;\dot{\theta}_a]^\mathrm{T},\\
		\bm{\omega}_e &= \bm{R}_a^e\bm{\omega}_a + [0\;\dot{\theta}_e\;0]^\mathrm{T},
	\end{aligned}
	\right.
\end{equation}
where $\dot{\theta}_a$ and $\dot{\theta}_e$ are the corresponding joint angular velocities, $\bm{\omega}_a$ and $\bm{\omega}_e$ represent the angular velocities of the azimuth and elevation platforms, and $\bm{\omega}_B=[p\;q\;r]^\mathrm{T}$ is the aircraft angular velocity, with $p, q$, and $r$ denoting the roll, pitch, and yaw rates, respectively. Here, $\bm{R}_B^a$ and $\bm{R}_a^e$ denote the rotation matrices. Since the camera is fixed to the elevation platform, $\bm{\omega}_e$ inherently describes the motion of the LOS in inertial space.

By the Newton--Euler formulation, the rotational dynamics of the two-axis ISP are given by
\begin{equation}\label{S-eq:newton_euler}
	\left\{
	\begin{aligned}
		\bm{I}_a\dot{\bm{\omega}}_a+\bm{\omega}_a\times(\bm{I}_a\bm{\omega}_a)&=\bm{T}_{a,total},\\
		\bm{I}_e\dot{\bm{\omega}}_e+\bm{\omega}_e\times(\bm{I}_e\bm{\omega}_e)&=\bm{T}_{e,total},
	\end{aligned}
	\right.
\end{equation}
where $\dot{\bm{\omega}}_a$ and $\dot{\bm{\omega}}_e$ denote the corresponding angular accelerations, $\bm{I}_a$ and $\bm{I}_e$ are the inertia tensors, and $\bm{T}_{a,total}$ and $\bm{T}_{e,total}$ denote the total torques acting on \textit{body A} and \textit{body E}, respectively.

Substituting \eqref{S-eq:kinematics} into \eqref{S-eq:newton_euler} and projecting the resulting equations onto the azimuth and elevation joint axes yields the manipulator-form dynamics:
\begin{equation}
	\label{S-eq:manipulator_form}
	\begin{bmatrix}
		M_{11} & M_{12} \\
		M_{21} & M_{22}
	\end{bmatrix}
	\begin{bmatrix}
		\ddot{\theta}_a \\
		\ddot{\theta}_e
	\end{bmatrix}
	+
	\begin{bmatrix}
		C_{11} & C_{12} \\
		C_{21} & C_{22}
	\end{bmatrix}
	\begin{bmatrix}
		\dot{\theta}_a \\
		\dot{\theta}_e
	\end{bmatrix}
	=
	\begin{bmatrix}
		\tau_a \\
		\tau_e
	\end{bmatrix}
	+
	\begin{bmatrix}
		d_a \\
		d_e
	\end{bmatrix},
\end{equation}
where $\ddot{\theta}_a$ and $\ddot{\theta}_e$ denote the joint angular accelerations, and $[\tau_a,\tau_e]^\mathrm{T}$ and $[d_a,d_e]^\mathrm{T}$ denote the control torque vector and disturbance torque vector, respectively. The matrix entries are defined as
\begin{equation}
	\label{S-eq:matrix_definitions}
	\begin{aligned}
		&\left\{\begin{aligned}
			M_{11} &= I_{za} + (I_{xe} - I_{ze} - m_e l_e^2)\sin^{2}\theta_e \\
			&\quad + (I_{ze} + m_e l_e^2), \\
			M_{12} &= M_{21} = 0, \\
			M_{22} &= I_{ye} + m_e l_e^{2},
		\end{aligned}\right. \\[10pt]
		&\left\{\begin{aligned}
			C_{11} &= (I_{xe} - I_{ze} - m_e l_e^2)\sin\theta_e\cos\theta_e\,\dot{\theta}_e, \\
			C_{12} &= (I_{xe} - I_{ze} - m_e l_e^2)\sin\theta_e\cos\theta_e\,\dot{\theta}_a, \\
			C_{21} &= -(I_{xe} - I_{ze} - m_e l_e^2)\sin\theta_e\cos\theta_e\,\dot{\theta}_a, \\
			C_{22} &= 0,
		\end{aligned}\right.
	\end{aligned}
\end{equation}
where $I_{za}$ is the moment of inertia of the azimuth platform about the $z$ axis, $I_{xe}$, $I_{ye}$, and $I_{ze}$ are the principal moments of inertia of the elevation platform about the $x$, $y$, and $z$ axes, respectively, $m_e$ denotes the mass of the elevation platform, and $l_e$ represents the distance from the elevation axis to its center of mass.
These expressions yield the manipulator-form dynamics \eqref{S-eq:manipulator_form} and state-space representation \eqref{eq:state_space_physical} used in the main paper.

\section{Detailed Disturbance Component Derivations}\label{S-S02}
This section provides the complete expressions for each disturbance component introduced in Section~2.2 of the main paper.
\textit{Gravitational Torque $f_{g,i}$}: Induced by the mass imbalance $m_e l_e$, the gravitational torques act on both axes and depend on the aircraft attitude as
\begin{equation}\label{S-eq:gravity_torque}
	\left\{
	\begin{aligned}
		f_{g,a} &= m_e g l_e \cos\theta_e (\sin\alpha \cos\theta_a + \sin\beta \sin\theta_a), \\
		f_{g,e} &= m_e g l_e \cos\theta_e \cos\beta \cos\alpha,
	\end{aligned}
	\right.
\end{equation}
where $g$ denotes the gravitational acceleration, and $\alpha$ and $\beta$ are the aircraft roll and pitch angles.

\textit{Gyroscopic Torque $f_{gyro,i}$}: From \eqref{S-eq:newton_euler}, the gyroscopic torques arise from projecting $\bm{\omega}_B$ onto the two gimbal axes and take the form
\begin{equation}\label{S-eq:gyroscopic_disturbance}
	\left\{
	\begin{aligned}
		f_{gyro,a} &= -(I_{xe} - I_{ye} + I_{ze}) (p\cos\theta_a + q\sin\theta_a) \\
		&\quad \times (-p\sin\theta_a + q\cos\theta_a), \\
		f_{gyro,e} &= -(I_{xe}-I_{ze}) (p\cos\theta_a + q\sin\theta_a) \\
		&\quad \times (-p\sin\theta_a + q\cos\theta_a)\sin\theta_e.
	\end{aligned}
	\right.
\end{equation}

In practice, although aircraft attitude and angular rates are available via the airborne communication bus~\cite{MIL-STD-1553B}, these measurements suffer from transmission latency and a low update rate (typically $50$ Hz), which is insufficient for high-rate ISP pointing control~\cite{Masten2008}. Consequently, the aircraft attitude and angular rates---and hence $f_{g,i}$ and $f_{gyro,i}$---are treated as unknown disturbances to the ISP.

\textit{Vibration Torque $f_{v,i}$}: A multi-harmonic model is adopted to describe the aircraft vibration disturbances~\cite{wangDisturbanceObservationSuppression2024}:
\begin{equation}\label{S-eq:harmonic_disturbance}
	\left\{
	\begin{aligned}
		f_{v,a} &= \gamma_{a,0} + \sum_{j=1}^{N_a} \gamma_{a,j} \sin(\Omega_{a,j} t + \varphi_{a,j}), \\
		f_{v,e} &= \gamma_{e,0} + \sum_{j=1}^{N_e} \gamma_{e,j} \sin(\Omega_{e,j} t + \varphi_{e,j}),
	\end{aligned}
	\right.
\end{equation}
where $\gamma_{i,0}$ denotes the constant vibration bias, $N_i$ denotes the number of dominant modes, $\gamma_{i,j}$ is the vibration amplitude, $\Omega_{i,j}$ is the harmonic frequency, and $\varphi_{i,j}$ is the phase offset. These modal parameters depend on the aircraft flight state. The APDO below retains the modes described by one common fundamental and its integer harmonics. Constant bias and modes outside this retained family are assigned to the residual component; a general set of unrelated modal frequencies is not identified by a single fundamental-frequency estimator.

\textit{Friction Torque $f_{d,i}$}: The LuGre model is adopted to describe the friction torque $f_{d,i}$~\cite{canudasdewitNewModelControl1995a}:
\begin{equation}\label{S-eq:lugre_model}
	\left\{
	\begin{aligned}
		f_{d,i} &= \sigma_{0,i} z_{d,i} + \sigma_{1,i} \dot{z}_{d,i} + \sigma_{2,i} \dot{\theta}_i, \\
		\dot{z}_{d,i} &= \dot{\theta}_i - \frac{|\dot{\theta}_i|}{g_i(\dot{\theta}_i)} z_{d,i},
	\end{aligned}
	\right.
\end{equation}
where $z_{d,i}$ is the internal bristle state, and $g_i(\dot{\theta}_i)$ describes the Stribeck effect governing velocity-dependent friction transitions. The parameters $\sigma_{0,i}$, $\sigma_{1,i}$, and $\sigma_{2,i}$ denote the bristle stiffness, micro-viscous damping, and viscous friction coefficient, respectively. Because these parameters are difficult to identify accurately on the ground and vary with ambient temperature and humidity, they are treated as unknown in the controller design.

\textit{Unmodeled Dynamics and Uncertainty $\delta_i$}: The residual term $\delta_i$ captures unmodeled dynamics and uncertainties such as cable restraint torques, higher-order coupling effects, and sensor biases.
These detailed expressions motivate the disturbance decomposition \eqref{eq:dist_group} into periodic and aperiodic components.

\section{Standard LSTM Gating Mechanisms}\label{S-S03}
The LSTM network used for aperiodic disturbance learning employs standard gating mechanisms. In $\Theta$, $\mathbf{W}$ and $\mathbf{R}$ denote the input and recurrent weight matrices, and $\mathbf{b}$ denotes the bias vectors, with subscripts $\mathrm{f}$, $\mathrm{in}$, $\mathrm{o}$, and $\mathrm{c}$ indicating the forget gate, input gate, output gate, and cell state, respectively. In the objective function, $M$ denotes the total number of training samples, and $\bm{d}_{r,k}^+$ denotes the reconstructed training target. The cell state $\mathbf{c}_k$ represents a slowly varying long-term memory state controlled by the forget gate $\mathbf{f}_k$ and the input gate $\mathbf{in}_k$, while the hidden state $\mathbf{h}_k$ captures a rapidly varying state that is updated step by step via recurrence relations during each training. $\varsigma(\cdot)$ denotes the element-wise sigmoid activation function and $\odot$ denotes the Hadamard product. The estimated disturbance $\bm{\hat{d}}_{\mathrm{LSTM},k}$ is obtained by a linear projection of the hidden state $\mathbf{h}_k$ without an additional activation function, because disturbance estimation is a continuous regression task. The output weight matrix is spectrally normalized as $\bar{\mathbf{W}}_{\mathrm{out}}=\mathbf{W}_{\mathrm{out}}/\|\mathbf{W}_{\mathrm{out}}\|_2$, where $\|\cdot\|_2$ denotes the spectral norm. This scaling is required for the stability analysis in Section~3.4 of the main paper. The LSTM-based disturbance learning optimization is:
In the objective function, $M$ denotes the total number of training samples, and $\bm{d}_{r,k}^+$ denotes the reconstructed training target for $\bm{d}_r$ in the $k$-th sample data, presented as $\bm{d}_{r,k}^+ = [d_{r,a,k},d_{r,e,k}]^\mathrm{T}$. It is obtained from the posterior $d^+_i$ and the APDO estimates during offline construction. The constraints in the following optimization describe the gating mechanisms and information flow of the LSTM network. The cell state $\mathbf{c}_k$ represents a slowly varying long-term memory state controlled by the forget gate $\mathbf{f}_k$ and the input gate $\mathbf{in}_k$, while the hidden state $\mathbf{h}_k$ captures a rapidly varying state that is updated step by step via recurrence relations during each training. $\varsigma(\cdot)$ denotes the element-wise sigmoid activation function and $\odot$ denotes the Hadamard product. Different from standard LSTM networks, we apply spectral normalization to the output layer weight matrix. The estimated disturbance $\bm{\hat{d}}_{\mathrm{LSTM},k}$ is obtained by a linear projection of the hidden state $\mathbf{h}_k$ without an additional activation function, because disturbance estimation is a continuous regression task. The output weight matrix is spectrally normalized as $\bar{\mathbf{W}}_{\mathrm{out}} = \mathbf{W}_{\mathrm{out}} / \| \mathbf{W}_{\mathrm{out}} \|_2$, where $\| \cdot \|_2$ denotes the spectral norm. This scaling is required for the stability analysis in Section~\ref{sec:stability}. The LSTM-based disturbance learning issue can be described in the optimization form:
\begingroup
\small
\begin{equation}\label{S-eq:lstm_optimization_explicit}
	\begin{aligned}
		&\Theta = \{ \mathbf{W}_{\mathrm{f}}, \mathbf{b}_{\mathrm{f}}, \mathbf{R}_{\mathrm{f}},
		\mathbf{W}_{\mathrm{in}}, \mathbf{b}_{\mathrm{in}}, \mathbf{R}_{\mathrm{in}}, \\
		& \quad\; \mathbf{W}_{\mathrm{c}}, \mathbf{b}_{\mathrm{c}}, \mathbf{R}_{\mathrm{c}},
		\mathbf{W}_{\mathrm{o}}, \mathbf{b}_{\mathrm{o}}, \mathbf{R}_{\mathrm{o}}, \\
		& \quad\; \bar{\mathbf{W}}_{\mathrm{out}}, \mathbf{b}_{\mathrm{out}} \}\\
		\min_{\Theta} \quad
		& J_{\mathrm{LSTM}}(\Theta) =
		\sum_{k=1}^{M} \left\| \bm{\hat{d}}_{\mathrm{LSTM},k}
		- \bm{d}_{r,k}^+ \right\|^2 \\[6pt]
		\text{s.t.} \quad
		& \mathbf{f}_k =
		\varsigma \bigl( \mathbf{W}_{\mathrm{f}} \bm{\psi}_k
		+ \mathbf{R}_{\mathrm{f}} \mathbf{h}_{k-1}
		+ \mathbf{b}_{\mathrm{f}} \bigr),\\
		& \mathbf{h}_k = \mathbf{o}_k \odot \tanh(\mathbf{c}_k), \\
		& \mathbf{in}_k =
		\varsigma \bigl( \mathbf{W}_{\mathrm{in}} \bm{\psi}_k
		+ \mathbf{R}_{\mathrm{in}} \mathbf{h}_{k-1}
		+ \mathbf{b}_{\mathrm{in}} \bigr), \\
		& \mathbf{c}_k
		= \mathbf{f}_k \odot \mathbf{c}_{k-1}
		+ \mathbf{in}_k \odot
		\tanh\!\bigl(
		\mathbf{W}_{\mathrm{c}} \bm{\psi}_k
		+ \mathbf{R}_{\mathrm{c}} \mathbf{h}_{k-1}
		+ \mathbf{b}_{\mathrm{c}} \bigr),\\
		& \mathbf{o}_k =
		\varsigma \bigl( \mathbf{W}_{\mathrm{o}} \bm{\psi}_k
		+ \mathbf{R}_{\mathrm{o}} \mathbf{h}_{k-1}
		+ \mathbf{b}_{\mathrm{o}} \bigr), \\
		& \bm{\hat{d}}_{\mathrm{LSTM},k}
		= \bar{\mathbf{W}}_{\mathrm{out}} \mathbf{h}_{k}
		+ \mathbf{b}_{\mathrm{out}}.
	\end{aligned}
\end{equation}
\endgroup
During online execution, the raw network output $\bm{\hat{d}}_{\mathrm{LSTM},k}$ is projected via \eqref{eq:lstm_proj} to bound its magnitude and rate.

\section{Detailed Discrete LSTM-DO Error Analysis}\label{S-S04}
This section gives the complete proof of Lemma~\ref{lem:do} in the main article.
\begin{proof}
Over a single sub-step, subtracting \eqref{eq:lstm_do_dynamics} from the exact plant increment and substituting $d-\hat d_{h,k}-\hat d_{r}=\tilde d_r+(d_h-\hat d_{h,k})$ yields $\tilde\omega^{+}=\tilde\omega+h(-l_0\tilde\omega+J^{-1}\tilde d_r)+hJ^{-1}(d_h-\hat d_{h,k})+r_\omega$, where $r_\omega=O(h^2)$ collects the Euler discretization remainder and the zero-order hold error of $\omega_k$, both of which are bounded via $\sup|\dot\omega|$ and $\sup|\ddot\omega|$ on the operating envelope. For the disturbance-error component, $\tilde d_r^{+}=e_L^{+}-z^{+}=(1-hl_2)\tilde d_r-hL_2\tilde\omega+\big[e_L^{+}-(1-hl_2)e_L\big]$, where the bracketed term accounts for the sampled increment and leakage of $e_L$; its leakage contribution is present at every sub-step when $l_2\ne0$. More explicitly, let $\bm S=\bm I+h\bm A_c$ and let $\bm r_{k,s}$ collect the forcing and remainder of sub-step $s$. Then $\bm\eta_k=\sum_{s=0}^{n_s-1}\bm S^{n_s-1-s}\bm r_{k,s}$ and $|\bm\eta_k|\le\sum_s|\bm S^{n_s-1-s}|\bar{\bm r}_{k,s}$ componentwise, whenever $|\bm r_{k,s}|\le\bar{\bm r}_{k,s}$. Expanding this finite sum gives \eqref{eq:eta_bound} and supplies computable constants once the sub-step envelopes are specified. Because $\bm A_d$ is Schur, $\bm\zeta_k=\bm A_d^k\bm\zeta_0+\sum_{l<k}\bm A_d^{k-1-l}\bm\eta_l$, and taking the second component via the triangle inequality gives \eqref{eq:lstm_do_bound}. For finite times, the bound includes $|[\bm A_d^{k-k_0}\bm\zeta_{k_0}]_2|$ in addition to $\varepsilon_d$; the $\limsup$ statement alone does not give an exact finite-time threshold. Element-wise gains are used rather than a single induced matrix norm because $\tilde\omega$ and $\tilde d_r$ have different physical units and a combined norm bound would be dominated by the inertia scaling $J^{-1}$.
\end{proof}

\section{Detailed Proofs of Lemmas~1 and~2}\label{S-S05}
The complete proofs of Lemmas~\ref{lem:freq} and~\ref{S-lem:pred} are reproduced below; the second statement is included to make the derivation self-contained.
\begin{proof}
(i) Substituting $y_k$ into \eqref{eq:apdo_freq_estimator}, the pre-projection error obeys $\tilde a^-_{k}=(1-g_k\phi_k)\big(\tilde a_{k-1}+a^\star_k-a^\star_{k-1}\big)-g_ke_k$. Because $P_{k-1}\in[p_{\min},p_{\max}]$, $g_k\phi_k=P_{k-1}\phi_k^2/(\lambda+P_{k-1}\phi_k^2)\in[0,1)$, and $|g_k|\le\max_{x}P_{k-1}x/(\lambda+P_{k-1}x^2)=\tfrac12\sqrt{P_{k-1}/\lambda}\le\bar g$. Since $a^\star_k\in\mathcal A$ and projection onto a closed interval is non-expansive, $|\tilde a_k|\le|\tilde a^-_k|$. Iterating over one window, all products of the factors $(1-g_l\phi_l)$ lie in $[0,1]$, and by $\ln(1-x)\le-x$ and $g_l\phi_l\ge p_{\min}\phi_l^2/(\lambda+p_{\max}\phi_{\max}^2)$ the full product is at most $\rho_a$. This gives the window inequality; iterating it yields the ultimate bound.
(ii) On $\mathcal A$, $\big|\tfrac{d}{da}\arccos(a/2)/(T_sD)\big|=\big(2T_sD\sqrt{1-a^2/4}\big)^{-1}\le L_f$.
(iii) Property (iii) follows from the projection and the bijection.
\end{proof}

For a slowly drifting frequency, $\bar\delta_a\le2T_sD\cdot T_s\delta_\Omega$ with $\delta_\Omega=\sup|\dot\Omega|$. Property (iii) holds even when the periodic disturbance vanishes. In that case, normalization amplifies noise, but the interval projection keeps $\hat\Omega_k$ within the admissible band and the amplitude projection below keeps the harmonic prediction bounded.

\textit{Harmonic prediction.} Let $\hat\psi_k=\hat\psi_{k-1}+\hat\Omega_kT_s$ and $\bm\varphi(\psi)=[\cos\psi,\sin\psi,\ldots,\cos n_h\psi,\sin n_h\psi]^{\mathrm T}$. Write the retained disturbance as $d_h(t)=\bm\varphi(\psi(t))^{\mathrm T}\bm b(t)$, where $\dot\psi=\Omega(t)$ and $\|\bm b(t)\|\le A_{\max}$. Rotating the coefficient vector to the phase reference $\hat\psi_k$ gives
$d_h(t_{k+j})=\bm\varphi(\hat\psi_k+\psi^\star_{k,j})^{\mathrm T}\bm\vartheta^\star_{k,j}$, with $\psi^\star_{k,j}=\int_{t_k}^{t_{k+j}}\Omega(t)\,dt$ and $\|\bm\vartheta^\star_{k,j}\|\le A_{\max}$. On a smooth segment, assume $\|\bm\vartheta^\star_{k,j}-\bm\vartheta^\star_{k,0}\|\le jT_s\delta_b$; constant amplitudes give $\delta_b=0$. The harmonic coefficients are fitted to $d^+_k$ by RLS, with the regressors passed through the same posterior low-pass filter. An additional constant regressor absorbs slow components but is excluded from prediction. The nonconstant coefficients are projected onto $\|\hat{\bm\vartheta}_k\|\le A_{\max}$, and
\begin{equation}\label{S-eq:apdo_q_filter}
\hat d_h(k+j|k)=\bm\varphi(\hat\psi_k+j\hat\Omega_kT_s)^{\mathrm T}\hat{\bm\vartheta}_k.
\end{equation}

\begin{lemma}[Prediction error]\label{S-lem:pred}
Let $\tilde{\bm\vartheta}_k=\bm\vartheta^\star_{k,0}-\hat{\bm\vartheta}_k$, $C_h=\sqrt{n_h}A_{\max}$ and $c_n=(\sum_{l=1}^{n_h}l^2)^{1/2}$. If the prediction interval is smooth as in Assumption~\ref{ass:A1}, then
\begin{equation}\label{S-eq:eps_h}
\begin{gathered}
|d_h(t_{k+j})-\hat d_h(k+j|k)|\le\varepsilon_h(k,j),\\
\begin{aligned}
\varepsilon_h(k,j)&:=\min\{B_h(k,j),\,2C_h\},\\
B_h(k,j)&=\sqrt{n_h}\big(\|\tilde{\bm\vartheta}_k\|+jT_s\delta_b\big)\\
&\quad+c_nA_{\max}\big(jT_s|\tilde\Omega_k|+\tfrac12(jT_s)^2\delta_\Omega\big).
\end{aligned}
\end{gathered}
\end{equation}
Across a frequency or coefficient switch, $\varepsilon_h(k,j)=2C_h$ remains valid. In particular, the error is uniformly bounded independently of convergence; the expression $B_h(k,j)$ need not itself be below $2C_h$.
\end{lemma}
\begin{proof}
Add and subtract the prediction using $\bm\vartheta^\star_{k,0}$ and the true phase. The coefficient and coefficient-variation terms are bounded by $\sqrt{n_h}(\|\tilde{\bm\vartheta}_k\|+jT_s\delta_b)$ because $\|\bm\varphi\|=\sqrt{n_h}$. The phase discrepancy is at most $jT_s|\tilde\Omega_k|+(jT_s)^2\delta_\Omega/2$ on a smooth interval. Harmonic $l$ contributes at most its amplitude times $l$ times this discrepancy; Cauchy--Schwarz gives $c_nA_{\max}$ times the phase bound. Independently, $|d_h|\le C_h$ and $|\hat d_h|\le C_h$, so the same error is at most $2C_h$. Taking the minimum proves \eqref{S-eq:eps_h}; only the latter amplitude argument is needed across switches.
\end{proof}
The coefficient error includes the effect of fitting the posterior signal in the presence of residual disturbances; projection bounds its magnitude without requiring exact separation or RLS convergence. Frequency and coefficient convergence~\cite{ioannouRobustAdaptiveControl1996} can tighten the local bound, whereas the amplitude bound alone supplies the uniform bound used below.

\section{Simulation Parameter Tuning Procedure}\label{S-S06}
This section records the simulation tuning and disjoint-seed evaluation procedure.
The observer and predictor parameters of each method were selected from a common budget of at most 16 candidates (DO bandwidth $\omega_o$ for MPC-DO; $\omega_o$ and Q-filter bandwidth for MPC-APDO; Kalman weights for MPC with fixed frequency and MPC with adaptive frequency; and $\omega_o$ and the RLS forgetting factor for MPC-APDO (preview) and MPC-DS). Each candidate was evaluated on seeds 0--2 across all four base frequencies $f_0$ and both scenarios, and a single parameter set per method and plant variant was retained. With $\omega_o$ and $\lambda$ of MPC-DS ($\kappa=0$) fixed, the drift-weight parameters were then selected from an additional grid of 15 candidates, $\kappa\in\{1,2,4,6,8\}$ and $L\in\{10,25,50\}$, evaluated on seeds 0--5 of the calibrated plant; the resulting pair $(\kappa,L)$ is held fixed across both plant variants and all mismatch levels. This procedure yielded $\omega_o=8$\,rad/s for MPC-DO and $\omega_o=10$\,rad/s for MPC-APDO, MPC-APDO (preview), and MPC-DS on the calibrated plant, and $\omega_o=56$\,rad/s for MPC-DO and MPC-DS on the nominal plant. All test evaluations below use the disjoint seeds 100--109, providing 40 runs per method and scenario.

\section{Observer Bandwidth Scaling and Element-Wise Gain Checks}\label{S-S07}

With the bandwidth parameterization~\cite{gaoScalingBandwidthParameterization2003} $l_0=2\omega_o$, $l_2=0$, $L_2=J\omega_o^2$, the matrix $\bm A_c$ has the double eigenvalue $-\omega_o$, $\bm I+h\bm A_c$ has the double eigenvalue $1-h\omega_o$, and the Schur condition reduces to $0<h\omega_o<2$. In this case the feedforward estimation error enters \eqref{eq:eta_bound} only through its increment: with $|d_{r,k+1}-d_{r,k}|\le T_s\delta_r$ and \eqref{eq:lstm_proj}, $|e_{L,k+1}-e_{L,k}|\le T_s\delta_e\le T_s(\delta_r+\delta_L)$, where $\delta_e$ is the rate bound of the feedforward estimation error, and, in the small-step regime $T_s\omega_o\ll1$, since $\gamma_{22}=2/(T_s\omega_o)+O(1)$ for $n_s=1$ while $\gamma_{21}$ is independent of $\omega_o$ to first order (it is fixed by $J$ and the double-pole placement),
\begin{equation}\label{S-eq:do_scaling}
	\begin{aligned}
\varepsilon_d&\approx\frac{2c_3\delta_e}{\omega_o}+\gamma_{21}c_1T_s\bar\varepsilon_{h,0}^{\mathrm{obs}}\\
&\quad+\gamma_{21}\bar\eta_1^{\mathrm{mod}}+\gamma_{22}\bar\eta_2^{\mathrm{mod}}+O(T_s).
\end{aligned}
\end{equation}
For the yaw-axis model of Section~\ref{sec:sim_setup} ($J=\tau_m/K$, $T_s=0.02$\,s, $n_s=4$), direct evaluation gives $T_s\gamma_{22}=0.207,\,0.079,\,0.043$ for $\omega_o=10,\,28,\,56$\,rad/s, against $2/\omega_o=0.200,\,0.071,\,0.036$, and $\gamma_{21}=6.7,\,6.6,\,5.9$, nearly independent of $\omega_o$ over these evaluated settings; the small-step approximation is not a uniform gain formula over the entire Schur range.

\section{Complete Proof of the Regional Practical Tracking Bound}\label{S-S08}

\begin{proof}
The Schur property makes all three gains finite. By \eqref{eq:eps_h_rho}, $|\bm K_d\delta\bm d_k|\le\varepsilon_h$. For $l=k-n_d$, nonexpansiveness of saturation onto an interval containing $d_{r,k}$ gives
$|d_{r,k}-\hat d^{\mathrm c}_{r,l}|\le|d_{r,l}-\hat d_{r,l}|+|D_l^\star-\Delta_l|\le\varepsilon_r$.
The finite-time term in $\bar\varepsilon_d$ follows from the Schur observer recursion; it includes the initial observer error.

Whenever the complete quadratic minimizer is feasible for the input box and the unsoftened predicted rows, its slack cost is zero and it minimizes the full objective. Both loops therefore use the affine law, and subtracting their nominal-history recursions gives
$\delta\bm z_{k+1}=\bm A_{\mathrm{cl}}\delta\bm z_k-\bm B\bm K_d\delta\bm d_k+\bm Ew_k+\bm\nu_k$,
with $|w_k|\le\varepsilon_r$. Iteration and the triangle inequality give \eqref{eq:iss} on every interval on which the affine law applies, and its right-hand side is at most $\mathcal B$.

At $k_0$, \eqref{eq:sequence_sensitivity} and \eqref{eq:region} place every component of $\bm V_{k_0}^0$ inside $[-b_0,b_0]\subseteq\mathcal I_{k_0}$ and every predicted state row inside its unsoftened bounds. The optimizer thus equals the affine law. The deviation recursion preserves the budget $\mathcal B$ at the next sample. Reapplying the same row inequalities proves this statement inductively for every $k\ge k_0$. The envelope inclusion ensures that the assumed residual bounds remain applicable.
\end{proof}

\section{Nominal Feedback and Impulse-Gain Verification}\label{S-S09}

\begin{remark}[Numerical check of the nominal feedback]\label{S-rem:numeric}
Direct least-squares evaluation for the retained-drift simulation model ($N_p=20$, $N_c=3$, $q_\omega=0$, $r=10^{-4}$, $p=1$) gives $\rho(\bm A_{\mathrm{cl}})=0.6202$ for yaw and $0.5979$ for pitch. For yaw, the angle-component impulse gains are $\sum_l|[\bm A_{\mathrm{cl}}^l\bm B]_1|=0.01383$ and $\sum_l|[\bm A_{\mathrm{cl}}^l\bm E]_1|=0.04400$\,deg per command level, and $\sum_j|K_{d,j}|=4.516$. The flight model uses its separately identified parameters, drift cancellation and $r=10^{-3}$; its spectral radii are $0.7788$ (yaw) and $0.7807$ (pitch), with corresponding angle-component gains $(0.03299,0.08764)$ and $(0.03211,0.06444)$. These checks establish the Schur property of the selected nominal feedbacks, not the full margin conditions of \eqref{eq:region} for a flight trajectory.
\end{remark}

%% file: main.bbl
\begin{thebibliography}{10}
\expandafter\ifx\csname url\endcsname\relax
  \def\url#1{\texttt{#1}}\fi
\expandafter\ifx\csname urlprefix\endcsname\relax\def\urlprefix{URL }\fi
\expandafter\ifx\csname href\endcsname\relax
  \def\href#1#2{#2} \def\path#1{#1}\fi

\bibitem{rueStabilizationPrecisionElectrooptical1969}
A.~K. Rue, Stabilization of precision electrooptical pointing and tracking
  systems, IEEE Trans. Aerosp. Electron. Syst. AES-5~(5) (1969) 805--819.
\newblock \href {https://doi.org/10.1109/TAES.1969.309879}
  {\path{doi:10.1109/TAES.1969.309879}}.

\bibitem{hilkertInertiallyStabilizedPlatform2008}
J.~M. Hilkert, Inertially stabilized platform technology concepts and
  principles, IEEE Control Syst. Mag. 28~(1) (2008) 26--46.
\newblock \href {https://doi.org/10.1109/MCS.2007.910256}
  {\path{doi:10.1109/MCS.2007.910256}}.

\bibitem{Masten2008}
M.~K. Masten, Inertially stabilized platforms for optical imaging systems, IEEE
  Control Syst. Mag. 28~(1) (2008) 47--64.
\newblock \href {https://doi.org/10.1109/MCS.2007.910201}
  {\path{doi:10.1109/MCS.2007.910201}}.

\bibitem{ekstrandEquationsMotionTwoaxes2001}
B.~Ekstrand, Equations of motion for a two-axes gimbal system, IEEE Trans.
  Aerosp. Electron. Syst. 37~(3) (2001) 1083--1091.
\newblock \href {https://doi.org/10.1109/7.953259}
  {\path{doi:10.1109/7.953259}}.

\bibitem{renAngularVibrationSurvey2003}
P.~Merritt, J.~Donaldson, D.~O'Brien, K.~Coleman, G.~Pyles, Angular vibration
  survey of various aircraft, in: Laser Systems Technology, Vol. 5087 of Proc.
  SPIE, 2003, p.~35.
\newblock \href {https://doi.org/10.1117/12.490123}
  {\path{doi:10.1117/12.490123}}.

\bibitem{maoModelingRobustContinuous2020}
J.~Mao, J.~Yang, X.~Liu, S.~Li, Q.~Li, Modeling and robust continuous {TSM}
  control for an inertially stabilized platform with couplings, IEEE Trans.
  Control Syst. Technol. 28~(6) (2020) 2548--2555.
\newblock \href {https://doi.org/10.1109/TCST.2019.2933382}
  {\path{doi:10.1109/TCST.2019.2933382}}.

\bibitem{yunInvestigationTwoStageVibration2020}
H.~Yun, L.~Liu, Q.~Li, H.~Yang, Investigation on two-stage vibration
  suppression and precision pointing for space optical payloads, Aerosp. Sci.
  Technol. 96 (2020) 105543.
\newblock \href {https://doi.org/10.1016/j.ast.2019.105543}
  {\path{doi:10.1016/j.ast.2019.105543}}.

\bibitem{dongDynamicModelControl2017}
F.~Dong, X.~Lei, W.~Chou, A dynamic model and control method for a two-axis
  inertially stabilized platform, IEEE Trans. Ind. Electron. 64~(1) (2017)
  432--439.
\newblock \href {https://doi.org/10.1109/TIE.2016.2608322}
  {\path{doi:10.1109/TIE.2016.2608322}}.

\bibitem{congIsolationControlInertially2016}
S.~Cong, K.~Deng, W.~Shang, D.~Kong, H.~Shen, Isolation control for inertially
  stabilized platform based on nonlinear friction compensation, Nonlinear Dyn.
  84~(3) (2016) 1123--1133.
\newblock \href {https://doi.org/10.1007/s11071-015-2557-4}
  {\path{doi:10.1007/s11071-015-2557-4}}.

\bibitem{tianAdaptiveModelPredictive2026}
K.~Tian, L.-J. Hu, C.~Bai, H.-P. Ren, Adaptive model predictive control with
  extended state observer for unmanned aerial vehicle trajectory tracking,
  Aerosp. Sci. Technol. 168 (2026) 111085.
\newblock \href {https://doi.org/10.1016/j.ast.2025.111085}
  {\path{doi:10.1016/j.ast.2025.111085}}.

\bibitem{nigusieExtendedStateObserver2026}
A.~T. Nigusie, A.~Z. Woldeyohannis, B.~Lin, Extended state observer enhanced
  constrained {LPV-MPC} for robust quadrotor trajectory tracking under wind
  disturbances, Aerosp. Sci. Technol. 177 (2026) 112671.
\newblock \href {https://doi.org/10.1016/j.ast.2026.112671}
  {\path{doi:10.1016/j.ast.2026.112671}}.

\bibitem{guoDisturbanceAttenuationRejection2005}
L.~Guo, W.-H. Chen, Disturbance attenuation and rejection for systems with
  nonlinearity via {DOBC} approach, Int. J. Robust Nonlinear Control 15~(3)
  (2005) 109--125.
\newblock \href {https://doi.org/10.1002/rnc.978} {\path{doi:10.1002/rnc.978}}.

\bibitem{chenDisturbanceobserverbasedControlRelated2015}
W.-H. Chen, J.~Yang, L.~Guo, S.~Li, Disturbance-observer-based control and
  related methods---{A}n overview, IEEE Trans. Ind. Electron. 63~(2) (2016)
  1083--1095.
\newblock \href {https://doi.org/10.1109/TIE.2015.2478397}
  {\path{doi:10.1109/TIE.2015.2478397}}.

\bibitem{maoAdaptiveNonsingularFast2022}
J.~Mao, S.~Li, C.~Xiang, Q.~Li, Adaptive nonsingular fast terminal sliding mode
  control for electromechanical actuator of inertial stabilized platform with
  actuator saturation, ISA Trans. 129 (2022) 379--392.

\bibitem{safaRobustOutputFeedback2019}
A.~Safa, R.~Yazdanpanah~Abdolmalaki, Robust output feedback tracking control
  for inertially stabilized platforms with matched and unmatched uncertainties,
  IEEE Trans. Control Syst. Technol. 27~(1) (2019) 118--131.
\newblock \href {https://doi.org/10.1109/TCST.2017.2761324}
  {\path{doi:10.1109/TCST.2017.2761324}}.

\bibitem{golestaniDisturbanceObserverConstrained2023}
M.~Golestani, W.~Zhang, Y.~Yang, N.~Xuan-Mung, Disturbance observer-based
  constrained attitude control for flexible spacecraft, IEEE Trans. Aerosp.
  Electron. Syst. 59~(2) (2023) 963--972.
\newblock \href {https://doi.org/10.1109/TAES.2022.3195691}
  {\path{doi:10.1109/TAES.2022.3195691}}.

\bibitem{leiCompositeControlMethod2015a}
X.~Lei, Y.~Zou, F.~Dong, A composite control method based on the adaptive
  {RBFNN} feedback control and the {ESO} for two-axis inertially stabilized
  platforms, ISA Trans. 59 (2015) 424--433.
\newblock \href {https://doi.org/10.1016/j.isatra.2015.09.011}
  {\path{doi:10.1016/j.isatra.2015.09.011}}.

\bibitem{changCompensationControlStrategy2024}
S.~Chang, J.~Cao, J.~Pang, F.~Zhou, W.~Chen, Compensation control strategy for
  photoelectric stabilized platform based on disturbance observation, Aerosp.
  Sci. Technol. 145 (2024) 108909.
\newblock \href {https://doi.org/10.1016/j.ast.2024.108909}
  {\path{doi:10.1016/j.ast.2024.108909}}.

\bibitem{zhuangMechatronicDisturbanceSuppression2026}
S.~Zhuang, J.~Li, J.~Deng, Y.~Mao, F.~Shen, Mechatronic disturbance suppression
  for electro-optical tracking in application to aerospace communication:
  {E}rror-based {ESO} and compensated backstepping design, Aerosp. Sci.
  Technol. 176 (2026) 112126.
\newblock \href {https://doi.org/10.1016/j.ast.2026.112126}
  {\path{doi:10.1016/j.ast.2026.112126}}.

\bibitem{miaoRobustObserverBasedVisual2027}
K.~Miao, C.~Wang, Y.~Niu, W.~Zhang, H.~Yu, A.~Wang, Robust observer-based
  visual servoing control of gimbal-mounted cameras for multi-{UAV} target
  tracking via relative dynamics and distributed {G}aussian processes, Aerosp.
  Sci. Technol. 180 (2027) 113689.
\newblock \href {https://doi.org/10.1016/j.ast.2026.113689}
  {\path{doi:10.1016/j.ast.2026.113689}}.

\bibitem{guoCompositeDisturbanceFiltering2023}
L.~Guo, W.~Li, Y.~Zhu, X.~Yu, Z.~Wang, Composite disturbance filtering: {A}
  novel state estimation scheme for systems with multisource, heterogeneous,
  and isomeric disturbances, IEEE Open J. Ind. Electron. Soc. 4 (2023)
  387--400.
\newblock \href {https://doi.org/10.1109/OJIES.2023.3317271}
  {\path{doi:10.1109/OJIES.2023.3317271}}.

\bibitem{sunNeuralNetworkBasedDOBC2017}
H.~Sun, L.~Guo, Neural network-based {DOBC} for a class of nonlinear systems
  with unmatched disturbances, IEEE Trans. Neural Netw. Learn. Syst. 28~(2)
  (2017) 482--489.
\newblock \href {https://doi.org/10.1109/TNNLS.2015.2511450}
  {\path{doi:10.1109/TNNLS.2015.2511450}}.

\bibitem{shenAdaptiveNeuralNetworkBased2021}
S.~Shen, J.~Xu, Adaptive neural network-based active disturbance rejection
  flight control of an unmanned helicopter, Aerosp. Sci. Technol. 119 (2021)
  107062.
\newblock \href {https://doi.org/10.1016/j.ast.2021.107062}
  {\path{doi:10.1016/j.ast.2021.107062}}.

\bibitem{leblebiciogluPhysicsInformedDisturbanceEstimation2022}
D.~Leblebicioglu, O.~Atesoglu, M.~Cakmakci, Physics-informed disturbance
  estimation and nonlinear controller design for a multi-axis gimbal system,
  IFAC-PapersOnLine 55~(37) (2022) 530--535.
\newblock \href {https://doi.org/10.1016/j.ifacol.2022.11.237}
  {\path{doi:10.1016/j.ifacol.2022.11.237}}.

\bibitem{jiaEVOLVEROnlineLearning2024}
J.~Jia, et~al., {EVOLVER}: {O}nline learning and prediction of disturbances for
  robot control, IEEE Trans. Robot. 40 (2024) 382--402.
\newblock \href {https://doi.org/10.1109/TRO.2023.3326318}
  {\path{doi:10.1109/TRO.2023.3326318}}.

\bibitem{huangLSTMMPCDeepLearning2023}
K.~Huang, K.~Wei, F.~Li, C.~Yang, W.~Gui, {LSTM-MPC}: {A} deep learning based
  predictive control method for multimode process control, IEEE Trans. Ind.
  Electron. 70~(11) (2023) 11544--11554.
\newblock \href {https://doi.org/10.1109/TIE.2022.3229323}
  {\path{doi:10.1109/TIE.2022.3229323}}.

\bibitem{rahamanSpectralBiasNeural2019}
N.~Rahaman, et~al., \href{https://proceedings.mlr.press/v97/rahaman19a.html}{On
  the spectral bias of neural networks}, in: Proc. 36th Int. Conf. Mach.
  Learn., Vol.~97 of Proc. Mach. Learn. Res., 2019, pp. 5301--5310.
\newline\urlprefix\url{https://proceedings.mlr.press/v97/rahaman19a.html}

\bibitem{guoAntidisturbanceControlTheory2014}
L.~Guo, S.~Cao, Anti-disturbance control theory for systems with multiple
  disturbances: {A} survey, ISA Trans. 53~(4) (2014) 846--849.
\newblock \href {https://doi.org/10.1016/j.isatra.2013.10.005}
  {\path{doi:10.1016/j.isatra.2013.10.005}}.

\bibitem{zhuDisturbanceSeparationBasedEnhanced2024}
Y.~Zhu, Z.~Bao, H.~Teng, Y.~Yang, Y.~Cui, Disturbance separation-based enhanced
  antidisturbance attitude control for flexible spacecrafts with composite
  disturbances, IEEE Trans. Ind. Electron. 71~(10) (2024) 13146--13156.
\newblock \href {https://doi.org/10.1109/TIE.2023.3344849}
  {\path{doi:10.1109/TIE.2023.3344849}}.

\bibitem{yangFixedTimeRefinedDisturbance2024}
Y.~Yang, Y.~Cui, Y.~Zhu, J.~Qiao, Fixed-time refined disturbance observer-based
  composite control for periscope-type coarse pointing assembly, IEEE Trans.
  Ind. Electron. 71~(11) (2024) 14896--14905.
\newblock \href {https://doi.org/10.1109/TIE.2024.3376831}
  {\path{doi:10.1109/TIE.2024.3376831}}.

\bibitem{baoDeepCoupledSlidingMode2025}
Z.~Bao, Y.~Zhu, H.~Teng, Y.~Cui, L.~Guo, Deep-coupled sliding mode disturbance
  observer-based composite attitude control for flexible communication
  spacecrafts, IEEE Trans. Ind. Electron. 72~(1) (2025) 833--844.
\newblock \href {https://doi.org/10.1109/TIE.2024.3398670}
  {\path{doi:10.1109/TIE.2024.3398670}}.

\bibitem{tengRefinedDisturbanceUtilization2025}
H.~Teng, Y.~Lu, P.~Xia, J.~Qiao, L.~Guo, Refined disturbance utilization-based
  green control for spacecraft with composite actuator disturbances, IEEE/ASME
  Trans. Mechatron. 30~(6) (2025) 5652--5662.
\newblock \href {https://doi.org/10.1109/TMECH.2025.3563132}
  {\path{doi:10.1109/TMECH.2025.3563132}}.

\bibitem{yuAntidisturbanceControllabilityAnalysis2021}
X.~Yu, Y.~Zhu, J.~Qiao, L.~Guo, Antidisturbance controllability analysis and
  enhanced antidisturbance controller design with application to flexible
  spacecraft, IEEE Trans. Aerosp. Electron. Syst. 57~(5) (2021) 3393--3404.
\newblock \href {https://doi.org/10.1109/TAES.2021.3079566}
  {\path{doi:10.1109/TAES.2021.3079566}}.

\bibitem{liuFixedTimeDOBPredictive2024}
Q.-Z. Liu, L.~Zhang, B.~Sun, Y.~Xiao, G.-W. Fan, Fixed-time disturbance
  observer-based attitude prescribed performance predictive control for
  flexible spacecraft, IEEE Trans. Aerosp. Electron. Syst. 60~(3) (2024)
  3209--3220.
\newblock \href {https://doi.org/10.1109/TAES.2024.3359591}
  {\path{doi:10.1109/TAES.2024.3359591}}.

\bibitem{Yang2012CCDC}
J.~Yang, S.~Li, Disturbance observer based model predictive control for
  autopilot design of bank-to-turn missiles, in: 2012 24th Chinese Control and
  Decision Conference (CCDC), 2012, pp. 2099--2104.
\newblock \href {https://doi.org/10.1109/CCDC.2012.6243029}
  {\path{doi:10.1109/CCDC.2012.6243029}}.

\bibitem{Maeder2009Automatica}
U.~Maeder, F.~Borrelli, M.~Morari, Linear offset-free model predictive control,
  Automatica 45~(10) (2009) 2214--2222.
\newblock \href {https://doi.org/10.1016/j.automatica.2009.06.005}
  {\path{doi:10.1016/j.automatica.2009.06.005}}.

\bibitem{Li2014JPC}
D.~Li, F.~Gao, Y.~Xi, Separated design of robust model predictive control for
  {LPV} systems with periodic disturbance, Journal of Process Control 24~(1)
  (2014) 250--260.
\newblock \href {https://doi.org/10.1016/j.jprocont.2013.10.010}
  {\path{doi:10.1016/j.jprocont.2013.10.010}}.

\bibitem{Xie2026Automatica}
Y.~Xie, X.~Yu, J.~Qiao, K.~Zhang, Y.~Shi, L.~Guo, Composite anti-disturbance
  control with disturbance prediction and utilization: A tube {MPC} approach,
  Automatica 184 (2026) 112700.
\newblock \href {https://doi.org/10.1016/j.automatica.2025.112700}
  {\path{doi:10.1016/j.automatica.2025.112700}}.

\bibitem{Ahmed2025LCSS}
S.~H. Ahmed, T.~Bonetti, L.~Fagiano, Periodic disturbance learning model
  predictive control, IEEE Control Systems Letters 9 (2025) 1826--1831.
\newblock \href {https://doi.org/10.1109/LCSYS.2025.3586633}
  {\path{doi:10.1109/LCSYS.2025.3586633}}.

\bibitem{Zhou2026TIE}
L.~Zhou, G.~Hu, W.~Xiao, M.~Li, J.~She, Spatial-repetitive-learning-based
  disturbance rejection model predictive control for rotational motor systems,
  IEEE Transactions on Industrial Electronics (2026) 1--11\href
  {https://doi.org/10.1109/TIE.2026.3734558}
  {\path{doi:10.1109/TIE.2026.3734558}}.

\bibitem{yangAdaptivePeriodicdisturbanceObserver2023}
Y.~Yang, Y.~Cui, J.~Qiao, Y.~Zhu, Adaptive periodic-disturbance observer based
  composite control for {SGCMG} gimbal servo system with rotor vibration,
  Control Eng. Pract. 132 (2023) 105407.
\newblock \href {https://doi.org/10.1016/j.conengprac.2022.105407}
  {\path{doi:10.1016/j.conengprac.2022.105407}}.

\bibitem{moratoDataScienceModel2024}
M.~M. Morato, M.~S. Felix, Data science and model predictive control: {A}
  survey of recent advances on data-driven {MPC} algorithms, J. Process Control
  144 (2024) 103327.
\newblock \href {https://doi.org/10.1016/j.jprocont.2024.103327}
  {\path{doi:10.1016/j.jprocont.2024.103327}}.

\bibitem{liSurveyControlLyapunov2023a}
B.~Li, S.~Wen, Z.~Yan, G.~Wen, T.~Huang, A survey on the control {L}yapunov
  function and control barrier function for nonlinear-affine control systems,
  IEEE/CAA J. Autom. Sinica 10~(3) (2023) 584--602.
\newblock \href {https://doi.org/10.1109/JAS.2023.123075}
  {\path{doi:10.1109/JAS.2023.123075}}.

\bibitem{fuAntiWindupMethodBased2024}
T.~Fu, L.~Guan, Y.~Gao, C.~Qin, An anti-windup method based on an {LADRC} for
  miniaturized inertial stabilized platforms on unmanned vehicles in marine
  applications, J. Mar. Sci. Eng. 12~(4) (2024) 616.
\newblock \href {https://doi.org/10.3390/jmse12040616}
  {\path{doi:10.3390/jmse12040616}}.

\bibitem{tengAdaptiveFixedTimeDisturbance2023}
H.~Teng, Z.~Wang, Y.~Zhu, X.~Yao, L.~Guo, Adaptive fixed-time disturbance
  observer-based proximity control for spacecrafts with dynamic obstacle
  constraints, IEEE Trans. Aerosp. Electron. Syst. 59~(6) (2023) 8521--8534.
\newblock \href {https://doi.org/10.1109/TAES.2023.3308556}
  {\path{doi:10.1109/TAES.2023.3308556}}.

\bibitem{kimQuadrotorConstraints2024}
S.-h. Kim, M.~Kim, Y.~Kim, Fault-tolerant adaptive control for trajectory
  tracking of a quadrotor considering state constraints and input saturation,
  IEEE Trans. Aerosp. Electron. Syst. 60~(3) (2024) 3148--3159.
\newblock \href {https://doi.org/10.1109/TAES.2024.3360021}
  {\path{doi:10.1109/TAES.2024.3360021}}.

\bibitem{MIL-STD-1553B}
{U.S. Department of Defense}, Digital time division command/response multiplex
  data bus, Tech. Rep. MIL-STD-1553B, U.S. Department of Defense, Washington,
  D.C., notice 2 (Sep. 1986).

\bibitem{wangDisturbanceObservationSuppression2024}
L.~Wang, et~al., Disturbance observation and suppression in an airborne
  electro-optical stabilized platform based on a generalized high-order
  extended state observer, Sensors 24~(11) (2024) 3629.
\newblock \href {https://doi.org/10.3390/s24113629}
  {\path{doi:10.3390/s24113629}}.

\bibitem{canudasdewitNewModelControl1995a}
C.~Canudas De~Wit, H.~Olsson, K.~Astrom, P.~Lischinsky, A new model for control
  of systems with friction, IEEE Trans. Autom. Control 40~(3) (1995) 419--425.
\newblock \href {https://doi.org/10.1109/9.376053}
  {\path{doi:10.1109/9.376053}}.

\bibitem{maApproximationbasedNussbaumGain2022}
H.~Ma, H.~Ren, Q.~Zhou, R.~Lu, H.~Li, Approximation-based {N}ussbaum gain
  adaptive control of nonlinear systems with periodic disturbances, IEEE Trans.
  Syst., Man, Cybern.: Syst. 52~(4) (2022) 2591--2600.
\newblock \href {https://doi.org/10.1109/TSMC.2021.3050993}
  {\path{doi:10.1109/TSMC.2021.3050993}}.

\bibitem{yangDOBasedAdaptiveConsensus2023}
B.~Yang, H.~Li, D.~Yao, W.~Meng, {DO}-based adaptive consensus control for
  multiple {MUAVs} with dynamic constraints, IEEE Trans. Syst., Man, Cybern.:
  Syst. 53~(4) (2023) 2387--2398.
\newblock \href {https://doi.org/10.1109/TSMC.2022.3213249}
  {\path{doi:10.1109/TSMC.2022.3213249}}.

\bibitem{ioannouRobustAdaptiveControl1996}
P.~A. Ioannou, J.~Sun, Robust Adaptive Control, Prentice Hall, Upper Saddle
  River, NJ, 1996.

\bibitem{gaoScalingBandwidthParameterization2003}
Z.~Gao, Scaling and bandwidth-parameterization based controller tuning, in:
  Proceedings of the 2003 American Control Conference, Vol.~6, Denver, CO, USA,
  2003, pp. 4989--4996.
\newblock \href {https://doi.org/10.1109/ACC.2003.1242516}
  {\path{doi:10.1109/ACC.2003.1242516}}.

\bibitem{kawatoInternalModelsMotor1999a}
M.~Kawato, Internal models for motor control and trajectory planning, Curr.
  Opin. Neurobiol. 9~(6) (1999) 718--727.
\newblock \href {https://doi.org/10.1016/S0959-4388(99)00028-8}
  {\path{doi:10.1016/S0959-4388(99)00028-8}}.

\bibitem{kerriganSoftConstraintsExact2000}
E.~C. Kerrigan, J.~M. Maciejowski, Soft constraints and exact penalty functions
  in model predictive control, in: Proceedings of the UKACC International
  Conference on Control (Control 2000), Cambridge, UK, 2000.

\bibitem{jiangInputtostateStabilityDiscretetime2001}
Z.-P. Jiang, Y.~Wang, Input-to-state stability for discrete-time nonlinear
  systems, Automatica 37~(6) (2001) 857--869.
\newblock \href {https://doi.org/10.1016/S0005-1098(01)00028-0}
  {\path{doi:10.1016/S0005-1098(01)00028-0}}.

\bibitem{limonInputtoStateStabilityUnifying2009}
D.~Limon, T.~Alamo, D.~M. Raimondo, D.~Mu{\~n}oz de~la Pe{\~n}a, J.~M. Bravo,
  A.~Ferramosca, E.~F. Camacho, Input-to-state stability: A unifying framework
  for robust model predictive control, in: L.~Magni, D.~M. Raimondo,
  F.~Allg{\"o}wer (Eds.), Nonlinear Model Predictive Control: Towards New
  Challenging Applications, Vol. 384 of Lecture Notes in Control and
  Information Sciences, Springer, Berlin, Heidelberg, 2009, pp. 1--26.
\newblock \href {https://doi.org/10.1007/978-3-642-01094-1_1}
  {\path{doi:10.1007/978-3-642-01094-1_1}}.

\end{thebibliography}
